\documentclass[12pt,reqno]{amsart}
\usepackage{latexsym, amsmath, amscd, amssymb, amsthm,longtable}
\usepackage[T2A]{fontenc}
\usepackage[utf8]{inputenc}
\usepackage[english]{babel}
\newtheorem*{theorem}{Тheorem}
\newtheorem*{corollary}{Corollary}

\usepackage{cite}
\usepackage{enumitem}
\setlist[itemize]{label=$\bullet$}
\usepackage{xcolor}
\usepackage{mathtools}

\usepackage{graphicx}
\makeatletter
\renewcommand{\paragraph}{\@startsection{paragraph}{4}{\z@}%
  {1.25ex \@plus 1ex \@minus .2ex}%
  {0.8ex}%
  {\normalfont\normalsize\bfseries}}
\makeatother
\usepackage{url}
\begin{document}

\title[Manifold Modeling]{Controlled Paraphrase Geometry in Sentence Embedding Space: Local Manifold Modeling and Latent Probing}
\author{Leonid Bedratyuk}
\address{ Khmelnytsky National University, Faculty of Information Technology, Ukraine}
\email{leonidbedratyuk@khmnu.edu.ua}
\maketitle

\begin{abstract}
The paper studies the local geometry of embedding clouds induced by 
\emph{controlled local classes of semantically close sentences}. The central 
question is how controlled paraphrase-like semantic variation is organized in 
sentence embedding space and whether the corresponding local structure can be 
explicitly modeled. The working assumption is that controlled semantic variation 
induces locally low-dimensional nonlinear structures in sentence embedding 
space, and that these structures can be approximated by low-degree fitted 
models.

We introduce a local geometric modeling scheme for embedding clouds based on 
affine, quadratic, and cubic fitted models. In addition, we use a surface-based 
latent probing procedure that constructs synthetic latent points in a reduced 
local PCA space with respect to the fitted local carrier. The proposed procedure 
is intended as an offline method for representation-space analysis, local 
manifold modeling, and geometry-aware latent probing.

The validity of generated latent points is evaluated using criteria that measure 
consistency with the fitted surface, preservation of neighborhood structure, 
agreement with the empirical distribution, stability of local second-order shape, and 
stability of the fitted-model coefficients. Experiments on controlled sets of 
semantically close sentences show that nonlinear local models describe embedding 
clouds more accurately than affine models. Surface-based generation provides 
strong fitted-geometry fidelity, including surface consistency, Hessian-based shape consistency, and coefficient consistency.

Downstream experiments further show that geometric validity of synthetic latent 
points does not automatically translate into improved classification 
performance. The results support explicit local geometric modeling of sentence 
embedding space and highlight the need to distinguish geometric validity from 
discriminative utility. As a resource contribution, we also introduce 
\textbf{CoPaGE-300K}, a controlled template-based dataset of semantically close 
sentence variants with slot-level annotations and precomputed sentence 
embeddings.
 
\vspace{0.5cm}
\end{abstract}

\section{Introduction}

Sentence embeddings have become a central representation format in modern NLP and
computational linguistics. They allow texts with different lexical and syntactic
forms to be compared, clustered, retrieved, and used in downstream machine
learning tasks. Despite their widespread use, the local geometry of sentence
embedding spaces remains only partially understood. In particular, it is still
unclear how controlled semantic variation is organized geometrically when a
family of semantically close sentences is encoded by a modern sentence embedding
model.

This paper studies this question in a controlled setting. We consider local
classes of semantically close sentences generated from template families with
explicitly defined lexical slots. Such classes are not treated as universal
linguistic categories. They are operationally defined local semantic
neighborhoods in which the source of variation is known and controlled. This
makes it possible to isolate local semantic variability from the uncontrolled
heterogeneity of large natural corpora and to study the corresponding embedding
clouds as geometric objects.

The central working hypothesis is that a controlled local class of semantically
close sentences forms an embedding cloud that has more structure than a simple
affine cluster. We assume that such a cloud is concentrated near a
low-dimensional nonlinear manifold-like patch. If this hypothesis is correct,
then the local structure of the cloud can be approximated by fitted geometric
models of low algebraic degree, such as affine, quadratic, or cubic models. The
purpose of this modeling is to approximate the effective local geometry of a
specific semantic patch, rather than the entire ambient embedding space.

The proposed approach is therefore best understood as a method for local
geometric modeling and latent probing in sentence embedding space. Its aim is to
investigate whether controlled semantic variation induces measurable local
geometry, whether this geometry can be explicitly modeled, and whether fitted
local carriers can be used constructively to generate diagnostic latent points.
In this setting, surface-based generation is not introduced as a universal
production mechanism for improving classification, but as a way to probe the
fitted local structure and examine which geometric properties are preserved.

The method proceeds as follows. A local embedding cloud is first projected to a
reduced PCA space that preserves the dominant variability of the cloud. This
dimensionality reduction step is not intended to claim that the full embedding
space is globally linear. It serves as a local stabilization procedure: low-variance
and unstable directions are discarded, while the principal directions of semantic
variation are retained. All polynomial fitting and surface-based projection
steps are then performed in this reduced \(r\)-dimensional space, rather than in
the original \(768\)-dimensional ambient embedding space. This makes explicit
local nonlinear modeling computationally feasible and avoids the combinatorial
explosion of polynomial coefficients in high dimension.

After the reduced local space is obtained, affine, quadratic, and cubic fitted
models are constructed. These models are used to analyze the local geometry of
the embedding cloud and to construct synthetic latent points based on the fitted
local carrier. The generation procedure is surface-based: an initial point is
constructed in the reduced space and then projected onto the fitted local
surface. The resulting point is finally reconstructed in the affine PCA subspace
of the original embedding space. In this way, generated latent points are tied to
the fitted geometry of the local semantic patch rather than obtained only by
linear interpolation or random perturbation.

It is important to distinguish three levels of analysis. The first level concerns
local geometric modeling: whether the embedding cloud is described more
accurately by nonlinear models than by an affine approximation. The second level
concerns geometric validity: whether generated latent points remain consistent
with the fitted local structure, including the surface, local second-order shape, and
model coefficients. The third level concerns discriminative utility: whether such
geometrically valid points are useful for a downstream classifier. A good local
geometric fit does not imply that the corresponding synthetic points will
improve classification performance. For this reason, the paper explicitly
separates geometric validity from downstream discriminative utility.

This distinction is central to the interpretation of the results. Surface-based
generation is designed to preserve fitted local geometry. Its behavior is
therefore evaluated using geometric criteria such as surface consistency,
Hessian-based shape consistency, and coefficient consistency. Neighborhood- and
distribution-based criteria are also reported, since they measure a different
objective: fidelity to the empirical density of the observed finite cloud.
Downstream classification is then used as an additional diagnostic test of
whether fitted-geometry fidelity translates into discriminative usefulness in a
low-data setting.

A separate contribution of this study is \textbf{CoPaGE-300K}
(\emph{Controlled Paraphrase Geometry Embeddings}), a controlled template-based
dataset of semantically close sentence variants with slot-level annotations and
precomputed sentence embeddings. The dataset contains more than
\(3 \cdot 10^5\) sentence instances across three template families and several
complexity regimes. It should not be interpreted as a natural paraphrase corpus.
Its value lies in its controlled design, full slot-level structure, and
suitability as a benchmark for studying local geometry, manifold-like structure,
and latent probing in sentence embedding spaces.

The contributions of the paper are as follows:
\begin{itemize}
    \item we formulate the problem of local geometric modeling of controlled
    paraphrase-like embedding clouds;
    \item we propose a surface-based latent probing procedure for constructing
    synthetic latent points in a reduced local PCA space;
    \item we introduce geometric validity criteria, including surface
    consistency, Hessian-based shape consistency, and coefficient consistency;
    \item we empirically show that fitted-geometry fidelity and downstream
    discriminative utility are distinct properties: geometrically valid
    synthetic points do not necessarily improve classification;
    \item we introduce \textbf{CoPaGE-300K}, a controlled template-based dataset
    of semantically close sentence variants with slot-level annotations and
    precomputed sentence embeddings.
\end{itemize}

The experimental study is organized around three research questions:
\begin{itemize}
    \item \textbf{RQ1.} Do controlled local classes of semantically close sentences form embedding clouds with measurable nonlinear geometry?
    \item \textbf{RQ2.} Can fitted local surfaces be used to generate synthetic latent points that preserve the fitted geometry of the original embedding cloud?
    \item \textbf{RQ3.} How do geometrically valid synthetic latent points affect downstream classification, and does their geometric validity coincide with discriminative utility?
\end{itemize}

The experiments are conducted on controlled template-based families of
semantically close sentences. This design allows us to isolate local semantic
variability and focus on the geometry of the corresponding embedding clouds. The
study compares affine and nonlinear local approximations, baseline latent
generation methods, and the proposed surface-based approach. The results show
that nonlinear fitted models better describe the local geometry than affine
models, that surface-based generation provides strong fitted-geometry fidelity,
and that downstream classification does not automatically benefit from
geometrically valid synthetic points.

The remainder of the paper is organized as follows.
Section~\ref{sec:related_work} discusses related work and positions the proposed
approach within semantic representation analysis, sentence embedding geometry,
and manifold-aware latent probing. Section~\ref{sec:methodology} presents the
local geometric modeling and surface-based generation methodology.
Section~\ref{sec:experimental_design} describes the dataset, experimental
design, baselines, and evaluation metrics. Section~$5$
reports and analyzes the experimental results. The Discussion section interprets
the findings, limitations, and possible applications beyond classification. The
final section summarizes the main conclusions.

%%%%%%%%%%%%%%%%%%%%%%%%%%%%%%%%%%%%%%%%%%%%%%%%%%
\section{Related work}
\label{sec:related_work}
%%%%%%%%%%%%%%%%%%%%%%%%%%%%%%%%%%%%%%%%%%%%%%%%%%%%%%%%

The study presented in this paper lies at the intersection of semantic
representation analysis, controlled textual variability, the geometry of
sentence embedding spaces, and local geometric modeling of embedding clouds.
Since the proposed method uses synthetic latent points as a diagnostic
construction, it is also related to representation-level generation and
manifold-aware methods. The main focus, however, is the local geometry induced by
controlled paraphrase-like semantic variation in sentence embedding space.

\subsection{Semantic representations, sentence meaning, and controlled similarity}

Recent work in computational linguistics increasingly treats vector
representations not only as engineering features for downstream tasks, but also
as objects of semantic analysis. Apidianaki~\cite{Apidianaki2023} surveys
approaches to word meaning representation and interpretation, including probing
studies aimed at identifying lexical semantic knowledge encoded in
contextualized representations. Chersoni et al.~\cite{Chersoni2021Decoding}
similarly address the interpretability of vector semantic spaces by relating
embeddings to explicit semantic features. These works motivate the view that
embedding spaces can be analyzed as structured semantic objects rather than as
opaque intermediate representations.

At the sentence level, semantic textual similarity and sentence meaning have
become central tools for evaluating representation models. Sun et
al.~\cite{Sun2022SentenceSimilarityContexts} propose a context-based view of
sentence similarity, where the meaning of a sentence is related to the contexts
in which it can occur. Wang et al.~\cite{Wang2023CollectiveSTS} further show
that semantic textual similarity judgments may involve substantial human
disagreement and should not always be reduced to a single deterministic scalar.
These works are relevant to the present study because they emphasize that
sentence similarity and semantic proximity are structured phenomena, rather than
merely numerical distances between isolated sentence vectors.

Controlled evaluation of sentence meaning is especially important for connecting
semantic variation with representation structure. Fodor et
al.~\cite{Fodor2025Compositionality} compare vector-based and syntax-based
models of sentence meaning on a challenging sentence similarity dataset, showing
the value of systematic sentence variation for evaluating semantic
representations. Amigó et al.~\cite{Amigo2022ICDS} develop an information
theory-based perspective on compositional distributional semantics and analyze
the relation between embedding spaces and meaning spaces. Large-scale lexical
similarity resources such as Multi-SimLex~\cite{Vulic2020MultiSimLex} further
illustrate the importance of controlled semantic similarity data for evaluating
representation models.

A related line of work studies whether neural representations encode specific
linguistic and semantic phenomena. Shwartz and Dagan~\cite{ShwartzDagan2019}
evaluate textual representations on lexical composition phenomena, including
meaning shift. Miletić and Schulte im Walde~\cite{MileticSchulte2024MWE}
investigate the semantics of multiword expressions in transformer-based
language models and show that compositionality information can be extracted from
such models under suitable conditions. Garí Soler et
al.~\cite{GariSoler2024WordSplitting} analyze how word splitting affects the
semantic content of contextualized word representations. More broadly, benchmark
resources such as Holmes~\cite{Waldis2024Holmes} reflect the growing interest in
diagnostic evaluation of linguistic competence in language models across syntax,
morphology, semantics, reasoning, and discourse.

The present work is aligned with this broader research direction but focuses on
a different object: the local geometry induced by controlled semantic variation
in sentence embedding space. Instead of evaluating semantic similarity only
through pairwise scores or downstream performance, we construct controlled
slot-based semantic classes and analyze the geometry of the corresponding
embedding clouds. This allows us to study whether controlled paraphrase-like
variation gives rise to reproducible local geometric structure.

\subsection{Textual variability and controlled semantic clouds}

Textual variability is commonly studied through paraphrase, lexical
substitution, contextual substitution, and data augmentation methods. For text
data, admissible transformations must preserve meaning, which makes them more
constrained than transformations commonly used for images. Existing methods
include character-level perturbations, noise-based transformations, lexical
substitutions, permutations, insertions, deletions, back-translation, contextual
substitution, and generation by large language models
\cite{Bayer2022, Kobayashi2018, WeiZou2019}.

S.~Kobayashi~\cite{Kobayashi2018} proposed \emph{contextual augmentation}, where
words are replaced with context-compatible alternatives generated by a language
model. This work is relevant here because it illustrates controlled local
semantic variability within a textual class. The EDA method proposed by J.~Wei
and K.~Zou~\cite{WeiZou2019} is simpler and relies on elementary transformations,
yet it also shows that local textual variability can affect model behavior.

In the present paper, textual variability is used primarily as a source of
controlled semantic clouds. The objective is not to generate new surface strings
as such, but to obtain reproducible local families of semantically close
sentences and to study their geometric image in sentence embedding space. This
is the point at which the present work departs from conventional text
augmentation: the object of analysis is the local embedding cloud associated
with controlled semantic variation.

\subsection{Geometry of sentence embedding spaces}

A separate line of research analyzes the geometry of embedding spaces induced by
pretrained language models. Ethayarajh~\cite{Ethayarajh2019} showed that
contextualized representations in BERT, ELMo, and GPT-2 exhibit strong
anisotropy, indicating that these spaces have nontrivial geometric structure.
For sentence embeddings, Li et al.~\cite{Li2020SentenceEmbeddings} demonstrated
that representations obtained from pretrained language models may have an
unfavorable anisotropic geometry, and proposed BERT-flow to transform their
distribution into a smoother isotropic space. These works show that the
geometry of representation space is not a secondary artifact, since it affects
semantic similarity and downstream behavior.

Closer to our setting is the work of Chu et
al.~\cite{Chu2023RefinedSBERT}, which proposes Refined SBERT, a method for
redescribing Sentence-BERT representations in a manifold space while preserving
local neighborhood structure. Even closer to the revised focus of the present
paper is the work of Tehenan~\cite{Tehenan2025SemanticGeometry}, where the
problem is formulated directly as the semantic geometry of sentence embeddings.
Our work continues this line while differing in two respects: we study
controlled template-based semantic clouds with known sources of variability, and
we explicitly construct low-degree fitted surfaces in a reduced PCA space for
analyzing fitted-geometry fidelity and latent probing.

\subsection{Local geometric modeling of data clouds}

In the broader context of data analysis and machine learning, there is a
research direction concerned with local geometric modeling of point clouds,
manifold approximation, the construction of local coordinate systems, and the
recovery of low-dimensional structure in data. Classical approaches of this kind
include PCA, LLE, Isomap, Laplacian Eigenmaps, t-SNE, UMAP, and related
dimensionality reduction methods \cite{McInnes2018, Sainburg2021}. These methods
provide the geometric apparatus required to formulate questions about the local
structure of high-dimensional data.

For our setting, these methods support the general validity of analyzing local
low-dimensional structure in high-dimensional embedding spaces. Classical
manifold learning methods, however, usually do not provide an explicit analytic
model of a local cloud suitable for testing fitted-geometry fidelity or for
constructing latent probes constrained by a fitted carrier. They motivate the
low-dimensionality hypothesis, while leaving open the problem of explicit
analytic modeling of local semantic patches.

This motivates local approximation of embedding clouds by low-degree surfaces,
including affine, quadratic, and cubic models. In this sense, our work uses ideas
from manifold learning as geometric background for constructing an analytically
specified local carrier. Such a fitted carrier makes it possible to analyze
proximity relations among points, the shape of a local semantic patch, the stability of local second-order shape, the change in fitted-model coefficients, and the
relation between geometric validity and downstream usefulness.

\subsection{Representation-level generation and manifold-aware methods}

Although the main focus of this paper is local geometry rather than training-set
augmentation, the proposed latent probing procedure is methodologically related
to feature-space and manifold-aware generation methods. Data augmentation is
widely treated as a regularization strategy for constructing new samples that
preserve class membership while expanding the support of the training
distribution \cite{Mumuni2022, Bayer2022, Longpre2020, DeVries2017}. Recent
approaches increasingly move from raw data transformations to the construction
of new examples in feature space or latent space
\cite{Mumuni2022, DeVries2017, Wang2022}.

In this context, Mixup constructs new examples as convex combinations of pairs
of training samples \cite{Zhang2018}. Manifold Mixup extends this idea to hidden
representations and shows that representation space can serve as a natural
environment for constructing intermediate points and regularizing a model
\cite{Verma2019}. These approaches are based on interpolation between already
available points. In contrast, the present work focuses on explicitly recovering
a local fitted carrier near which the embedding cloud is located.

Other manifold-aware methods incorporate local geometry more explicitly. UMAP
Mixup regularizes the hidden space using UMAP before mixing points
\cite{ElLaham2024}. OMADA constructs on-manifold adversarial samples using a
generative model and latent perturbations \cite{OMADA2021}. Paschali et
al.~\cite{Paschali2019} explore geometric transformations along the boundaries
of class manifolds. TextManiA~\cite{Moon2023} uses text-driven manifold
augmentation for visual features and shows that semantic directions in
text-embedding space can have constructive value in another feature space.

These approaches are relevant because they share the assumption that useful
operations in representation space should respect local manifold-like structure.
The present paper takes a different step: it explicitly recovers a local
analytic carrier in the form of a fitted low-degree surface and uses generated
latent points as probes of this fitted geometry. This makes it possible to
separate fitted-geometry fidelity, empirical-density fidelity, and downstream
discriminative utility.

\subsection{Comparison with existing approaches and the gap in the literature}

The existing literature provides important components for the present study:
semantic representation analysis, sentence similarity, geometry of embedding
spaces, and manifold-aware operations in latent spaces. However, the setting
considered here remains insufficiently explored. We focus on controlled
paraphrase-like variation and treat the resulting sentence embedding clouds as
local geometric objects.

In particular, it remains open whether controlled semantic variation induces
reproducible low-dimensional nonlinear structure in sentence embedding space,
whether such structure can be explicitly approximated by low-degree fitted
carriers, and whether generated latent points can be used as probes of this
local geometry. This configuration---controlled semantic cloud, explicit local
low-order fitting, fitted-geometry fidelity, and comparison with downstream
discriminative utility---defines the main gap addressed by the present paper.

The contribution of this work lies in separating three questions that are often
treated together in representation-level generation and augmentation studies:
whether a local embedding cloud has a reproducible geometric structure; whether
one can construct latent points consistent with that structure; and whether such
points provide additional discriminative information for a specific downstream
task.

\subsection{Positioning of the proposed approach}

The proposed approach belongs to the broader area of representation-space
analysis and modeling, and more specifically to the study of controlled
paraphrase geometry in sentence embedding space. Its closest methodological
neighbors include Manifold Mixup \cite{Verma2019}, UMAP Mixup
\cite{ElLaham2024}, and OMADA \cite{OMADA2021}, since these approaches also rely
on the assumption that useful operations in representation space should respect
the local manifold-like structure of the data. The key difference is that our
approach explicitly constructs a fitted local model of the embedding cloud by
means of low-degree surfaces, rather than using manifold structure implicitly,
through interpolation, or through a generative mechanism.

Conceptually, the proposed approach is close to first-order manifold-based
methods, since it assumes that the local structure of data in representation
space is low-dimensional. The present setting, however, goes beyond local linear
or tangent approximation. First-order approaches typically rely on a locally
linear structure of a manifold-like patch, whereas in our work the local carrier
is explicitly modeled by quadratic fitted models and, in selected cases, by
cubic fitted models as well. In this sense, the proposed approach can be viewed
as higher-order local geometric modeling of sentence embedding clouds.

The role of latent generation in this work is diagnostic rather than primarily
production-oriented. Surface-based generation is used to test whether the fitted
local structure can serve as a constructive carrier for new latent points and to
examine which aspects of the local geometry are preserved. Downstream
classification is then used as an additional diagnostic test, aimed at
distinguishing fitted-geometry fidelity from discriminative utility. This
distinction is central to the paper: a latent point may be geometrically valid
with respect to a local semantic patch without necessarily improving a particular
classifier.

Thus, our work connects three research lines: controlled textual variability and
paraphrase-like transformations, semantic representation analysis, and local
geometric modeling of embedding clouds. The novelty lies in treating controlled
semantic variation as a source of measurable local geometry in sentence embedding
space, explicitly approximating this geometry by low-degree fitted carriers, and
using latent probing to distinguish local geometric validity from downstream
discriminative utility.

%%%%%%%%%%%%%%%%%%%%%%%%%%%%%%%%%%%%%
\section{Methodology}
\label{sec:methodology}
%%%%%%%%%%%%%%%%%%%%%%%%%%%%%%%%%%%%%%%%

This section formalizes the proposed framework for local geometric modeling of controlled semantic classes in sentence embedding space and describes the surface-based procedure used to generate geometrically consistent synthetic latent points.

\subsection{Problem statement}

Let
$$
\mathcal{S}=\{s_1,\dots,s_N\}
$$
be a finite set of texts, which we treat as a \emph{controlled local semantic class}. 
Such a class is not required to consist of strict paraphrases in a strong linguistic 
sense \cite{BhagatHovy2013}; it may instead be a controlled family of semantically 
close sentences with a known source of variation. This choice makes it possible to 
isolate local semantic variability and study its geometric image in embedding space.

Let
$$
E:\mathcal{S}\to\mathbb{R}^d
$$
be an embedding map that assigns to each text a vector representation in feature 
space. Then the set
$$
\mathcal{X}=E(\mathcal{S})=\{x_1,\dots,x_N\}\subset\mathbb{R}^d
$$
forms a local point cloud representing the given class in the chosen embedding space.

The goal of this work is to construct new points
$$
x^\ast\in\mathbb{R}^d,\qquad x^\ast\notin\mathcal{X},
$$
which can be interpreted as \emph{synthetic latent examples} of the same local 
semantic class. In contrast to methods that generate new surface strings, the construction in
this work takes place directly in embedding space. The task is therefore to construct geometrically consistent 
synthetic points in feature space, instead of generating new texts as such.

It is essential to distinguish two related but non-equivalent tasks. The first is 
to describe the local embedding cloud adequately by means of an explicit geometric 
model. The second is to use this model to construct new latent points that preserve 
the local structure of the class and may be useful for latent probing or downstream analysis. 
A good approximation of the original cloud is a necessary condition, although by 
itself it does not guarantee the quality of the generated points.

\subsection{Local manifold hypothesis}

The working assumption is the \emph{local manifold hypothesis}: the point cloud 
$\mathcal{X}$ is not an arbitrary scattered subset of $\mathbb{R}^d$, but has a 
locally low-dimensional geometric structure. We do not assume the existence of a 
predefined ``true manifold''; instead, we work with an empirical 
\emph{manifold-like carrier} that can be approximated by explicit low-order models.

The term \emph{manifold-like} is used here in an operational sense. We do not
assume the existence of a true underlying semantic manifold. The working
hypothesis is that controlled slot-based variation restricts the active semantic
degrees of freedom, and therefore its image in sentence embedding space may have
low effective dimensionality and a regular local shape. The fitted surface should
therefore be understood as an explicit local carrier used to approximate and
probe the observed embedding cloud, rather than as the true semantic manifold.

The basic nonlinear model in this work is a quadric surface, since it is the 
simplest form capable of capturing local second-order shape. A cubic model is considered 
as an additional extension for cases in which the local embedding cloud exhibits 
more complex geometry and quadratic approximation becomes insufficient.

The aim of the study is to test whether the local geometry of controlled semantic 
classes has a regular structure suitable for explicit low-parametric modeling.

This setting is naturally related to the ideas of \textit{Mixup} and 
\textit{Manifold Mixup}. In classical Mixup, new examples are generated along the 
chord between two points. In \textit{Manifold Mixup}, the same idea is transferred 
to the space of hidden representations. In both cases, interpolation remains 
linear. Our setting is based on the observation that the local carrier of a 
semantic class may be curved, so a chord need not be consistent with its local 
geometry. For this reason, we consider the construction of new points followed by 
their alignment with a fitted local surface, rather than pure linear 
interpolation.

\subsection{Local dimensionality reduction}

Since direct surface approximation in the full $d$-dimensional space is usually 
unstable, local dimensionality reduction is performed first. Let
$$
P:\mathbb{R}^d\to\mathbb{R}^r,\qquad r\ll d,
$$
be a projection map constructed by PCA on the set $\mathcal{X}$. Denote
$$
z_i=P(x_i),\qquad i=1,\dots,N.
$$
This gives the cloud
$$
Z=\{z_1,\dots,z_N\}\subset\mathbb{R}^r,
$$
on which the local geometric model is constructed.

Here PCA serves as a stabilizing and coordinate-defining step, allowing the 
analysis to move to a space in which the principal directions of local variation 
are expressed most strongly. All subsequent fitted surface modeling is performed 
in $\mathbb{R}^r$. Geometric statements formulated in the reduced space should therefore be
understood as statements about the effective local coordinate representation of
the cloud, rather than about the full ambient embedding space.

In the experiments, the reduced dimension \(r\) was selected separately for each
local class using the \(90\%\) explained-variance criterion. Thus, \(r\) is not a
global hyperparameter fixed for all regimes, but an adaptive local quantity
determined by the effective dimensionality of the corresponding embedding cloud.
Across the fifteen regimes used in the study, the selected values ranged from
\(r=10\) to \(r=37\). Small one-slot regimes required between \(10\) and \(12\)
components, whereas the full four-slot regimes required between \(36\) and \(37\)
components. This supports the interpretation of PCA as a local stabilization
procedure rather than as a global linear model of the full embedding space.

\subsection{Quadratic model of the local patch}

In $\mathbb{R}^r$, a general quadric is given by
$$
f(z)=z^\top Qz+L^\top z+C=0,
$$
where
$$
Q\in\mathbb{R}^{r\times r},\qquad Q^\top=Q,\qquad
L\in\mathbb{R}^r,\qquad C\in\mathbb{R}.
$$
The number of independent parameters of this model is
$$
p(r)=\frac{r(r+1)}{2}+r+1.
$$

For each point $z_i\in Z$, we obtain one linear equation in the coefficients of 
the quadric. After expanding the quadratic and linear terms, this yields a 
homogeneous system
$$
A\theta\approx 0,
$$
where $\theta$ is the vector of unknown model parameters. The fitted quadric is 
defined as the singular vector of the matrix $A$ corresponding to the smallest 
singular value, under the normalization condition $\|\theta\|=1$. This gives a 
stable algebraic scheme for parameter estimation.

It should be emphasized that this formulation minimizes the \emph{algebraic 
residual}, rather than the orthogonal geometric distance to the surface in the 
strict sense. This makes the procedure computationally convenient and stable, 
while the fitted residual should not be identified with the exact geometric 
distance.

If the quadratic model does not describe the local cloud with sufficient accuracy, 
the same scheme can be extended to surfaces of third order. A cubic model has more 
coefficients and is therefore less stable in the small-data regime or when the 
local sample is not sufficiently dense. It increases the flexibility of the 
approximation, but at the same time raises the risk of overfitting to the 
empirical configuration of points.

In this work, no separate regularization of the cubic fitted model is introduced. 
This risk is mitigated by using controlled local classes of sufficiently large 
size, adaptive selection of the reduced-space dimensionality, and subsequent 
evaluation of generated points by geometric validity criteria, in particular 
Hessian-based shape consistency and coefficient consistency. Therefore, the cubic model is 
used only when its appropriateness is supported empirically.

\subsection{Generation of new latent points}

After the fitted surface has been constructed, new synthetic points are generated
in two stages.

\subsubsection{Internal barycentric initialization}

First, an internal point is constructed as a convex combination of the existing
points in the cloud:
$$
v=\sum_{i=1}^{N}\lambda_i z_i,
\qquad
\lambda_i\ge 0,
\qquad
\sum_{i=1}^{N}\lambda_i=1.
$$
Thus, \(v\) belongs to the convex hull of the local cloud. This guarantees that
the initial point is internal in the affine-convex sense, but it does not by
itself guarantee stable angular behavior with respect to the origin. This
distinction is important because the subsequent projection step is performed in
a reduced representation space, where directions and local geometry may be
sensitive to points whose convex hull passes close to the origin.

To make this issue explicit, let
$$
D_Z=\operatorname{diam}(Z)
=
\max_{i,j}\|z_i-z_j\|
$$
be the Euclidean diameter of the local cloud
\(Z=\{z_1,\dots,z_N\}\subset\mathbb{R}^r\). More generally, for a reference point
\(a\in\mathbb{R}^r\), define
$$
d_Z(a)=\operatorname{dist}(a,\operatorname{conv}(Z)).
$$
If \(d_Z(a)>0\), then the convex hull of the cloud is separated from the
reference point \(a\). The following elementary stability statement shows that
barycentric points are angularly controlled with respect to \(a\), provided that
the cloud diameter is small relative to this separation.

\begin{theorem}
Let \(Z=\{z_1,\dots,z_N\}\subset\mathbb{R}^r\) be a finite cloud and let
\(a\in\mathbb{R}^r\) be a reference point such that
$$
d_Z(a)=\operatorname{dist}(a,\operatorname{conv}(Z))>0.
$$
Let \(v\in\operatorname{conv}(Z)\) be a barycentric point with \(v\ne a\). If
$$
D_Z\leq 2d_Z(a),
$$
then, for every \(j=1,\dots,N\),
$$
\angle(v-a,z_j-a)
\leq
2\arcsin\frac{D_Z}{2d_Z(a)}.
$$
\end{theorem}

The proof is given in Appendix~\ref{app:barycentric_stability}.The theorem shows that barycentric initialization is angularly controlled when
the local cloud is well separated from the chosen reference point \(a\) relative
to its diameter. In this case, a barycentric point cannot move into an unstable region near the
chosen reference point and remains within a controlled angular neighborhood of
the original cloud. This provides a
geometric reason for applying barycentric initialization locally, and, when the
cloud has a more complex structure, within a small local subcluster rather than
over the entire set of points.

In the practical implementation, the coefficients \(\lambda_i\) are sampled from
a Dirichlet distribution on the standard simplex:
$$
(\lambda_1,\dots,\lambda_N)\sim \mathrm{Dir}(\alpha_1,\dots,\alpha_N),
\qquad
\lambda_i\ge 0,\qquad
\sum_{i=1}^{N}\lambda_i=1.
$$
In the simplest case, a symmetric Dirichlet distribution with identical
parameters \(\alpha_i=\alpha\) is used. This provides a random but controlled
sampling of internal barycentric combinations. In the baseline scenario, we use
the symmetric Dirichlet distribution with \(\alpha_i=1\). Since distribution
consistency is later evaluated with respect to the empirical local distribution,
the distribution of barycentric weights is fixed explicitly. The resulting
barycentric point serves as the initial approximation to a new latent example,
which is then aligned with the fitted local carrier.

In the standard centered PCA coordinates used in our implementation, the origin
coincides with the empirical mean of the reduced cloud. Since this mean is an
equal-weight convex combination of the points, the origin belongs to
\(\operatorname{conv}(Z)\), and therefore the above separation condition does
not hold for \(a=0\). Thus, the theorem is not used as an automatic guarantee
for centered PCA coordinates. Its role is to clarify an additional sufficient
condition under which barycentric initialization is angularly stable relative to
an external reference point. In the implemented algorithm, the unconditional
guarantee is convex-hull internality; angular and distributional admissibility
are checked empirically by the geometric validity criteria.

\subsubsection{Alignment with the fitted surface}

After the initial point $v$ has been constructed, it is locally aligned with the
fitted surface. Let the fitted surface be given by the implicit equation
$$
f(z)=0,
\qquad z\in\mathbb{R}^r.
$$
The starting point of the iterative process is
$$
z_0=v.
$$
To project the point onto the fitted surface, we use a fast stabilized
Newton-type scheme:
$$
z_{k+1}
=
z_k
-
\frac{f(z_k)}
{\|\nabla f(z_k)\|^2+\varepsilon}
\,\nabla f(z_k),
\qquad \varepsilon>0.
$$
After several iterations, we obtain a point
$$
z^\ast\in\mathbb{R}^r,
$$
which either belongs to the fitted surface or is sufficiently close to it in
terms of the normalized residual
$$
\frac{|f(z^\ast)|}{\|\nabla f(z^\ast)\|+\varepsilon}.
$$

In the practical implementation, the length of a single iterative step is also
bounded. If the update vector exceeds a prescribed local threshold, its norm is
rescaled to that threshold. This prevents excessively large jumps in the reduced
PCA space and preserves the local character of the projection.

The iterative process stops when at least one of the following conditions is
satisfied:
$$
|f(z_k)|<10^{-6},
\qquad
\|\nabla f(z_k)\|<10^{-8},
\qquad
k\ge 50.
$$
The first condition corresponds to a sufficiently small residual with respect to
the fitted surface, the second prevents unstable updates in regions with an
almost vanishing gradient, and the third sets the maximum number of iterations.

In the reported experiments, the projection procedure was stable for the
evaluated full-regime classes. For \(A\text{-}C5\), \(B\text{-}C5\), and
\(C\text{-}C5\), the mean normalized projection residuals were respectively
\(9.22\cdot 10^{-8}\), \(5.72\cdot 10^{-8}\), and \(8.98\cdot 10^{-8}\), with a
convergence fraction equal to \(1.0\) in all three cases. Thus, for the main
surface-based generation experiment, the Newton-type projection did not exhibit
observable convergence failures.

The point is then mapped back from the reduced PCA space to the original
embedding space by reconstruction in the affine subspace spanned by the first
$r$ principal components:
$$
x^\ast \approx \mu + U_r z^\ast,
$$
where $\mu$ is the center of the local cloud and $U_r$ is the matrix of the first
$r$ principal directions. The resulting vectors
$$
x^\ast\in\mathbb{R}^d
$$
are treated as synthetic embeddings associated with the original local semantic
class.

This alignment procedure generates a new point in a neighborhood of the existing
embedding cloud and explicitly moves it onto the fitted local carrier. This
projection is what distinguishes surface-based generation from simple
barycentric generation or linear interpolation in the reduced space.

\subsection{Several algorithmic variants of generation}

The proposed scheme admits several implementation variants.

The first variant uses a global choice of barycentric weights over the entire
local cloud. This is the simplest implementation and works well for compact and
relatively homogeneous clouds.

The second variant first performs local grouping of points and then constructs
barycentric combinations only within a small subcluster. This scheme is more
stable when the initial set contains several distinct local patches.

The third variant replaces the quadratic fitted surface with a cubic fitted
surface when experiments indicate that a second-order surface does not reproduce
the local structure with sufficient accuracy.

All these variants belong to the same methodology: a synthetic point is first
constructed as an internal combination of the available data and is then aligned
with the local geometric model of the manifold-like patch.
\subsection{Geometric validity criteria}

To evaluate the proposed generation procedure, we use geometric criteria that
test whether a generated latent point remains consistent with the local structure
estimated from the original embedding cloud.

Since a synthetic point does not necessarily correspond to an explicitly existing
natural-language text, its validity in this work is defined primarily in
geometric terms. We distinguish three primary aspects of such validity: consistency with the
fitted surface, consistency with the local neighborhood structure, and
consistency with the empirical local distribution. In addition, when synthetic
points are added to the original cloud and the local model is refitted, we also
measure the stability of the Hessian-based second-order descriptors and the
stability of the fitted-model coefficients.

\subsubsection{Consistency with the fitted surface}

A new point should remain close to the same local surface that approximates the
original cloud. This is evaluated by the algebraic residual
$$
|f(z^\ast)|
$$
or by a normalized surface residual computed with respect to the fitted implicit model. Smaller values of this
criterion indicate that the generated point preserves the selected geometric
model of the local class.

\subsubsection{Local neighborhood consistency}

A new point should be naturally embedded into the local structure of the cloud.
For this purpose, we analyze its distances to the nearest neighbors and the
stability of the local $k$-neighborhood after adding the synthetic point. In the
practical implementation, this criterion is evaluated by a normalized overlap
measure between sets of $k$-nearest neighbors, as well as by average local
distances in the reduced space. The parameter $k$ is chosen as a small number
corresponding to the local scale of the cloud; in our experiments, it is fixed in
advance and used uniformly for all classes to ensure comparability of the
results. The resulting metric is therefore local and dimensionless, and larger
consistency values indicate better preservation of the neighborhood structure.

\subsubsection{Consistency with the local distribution}

A synthetic point should not be a geometric outlier with respect to the local
semantic class. As a measure of deviation from the empirical local distribution,
we use the Mahalanobis distance and its local analogue in $\mathbb{R}^r$. This
criterion makes it possible to filter out points that formally lie near the
fitted surface but behave statistically as outliers.

Together, these criteria define the geometric admissibility of generated latent
points: a point is considered valid if it satisfies the fitted-surface
constraint, does not distort the local neighborhood structure, and remains within
the empirical range of the local manifold-like patch. The refit-based criteria
then assess whether adding such points changes the fitted local carrier itself.

\subsection{Methodological meaning of the approach}

The proposed scheme treats embedding space not only as an environment for
measuring semantic similarity or training downstream models, but also as an
object of local geometric analysis. The central question of this work is whether
an embedding cloud induced by controlled semantic variability has a reproducible
local structure, and whether this structure can be explicitly described by a
low-degree fitted model. Accordingly, the main emphasis shifts from constructing
additional training examples to recovering a local geometric carrier that
approximates the corresponding semantic patch.

In this sense, the proposed approach differs both from text-based methods that
operate in the space of symbolic sequences and from interpolation-based methods
that construct new representations along straight segments between existing
points. In our setting, the local embedding cloud is analyzed first; then a fitted
surface is constructed in the reduced PCA space; only after that are synthetic
latent points considered, with alignment to this local geometric structure.

Thus, the proposed approach should primarily be understood as a method for local
geometric modeling of semantic variation in sentence embedding space.
Geometry-aware latent generation in this scheme is neither an end in itself nor a
guarantee of improved downstream classification. It is a way to test how far the
fitted local structure can serve as a constructive carrier for generating new
latent representations. For this reason, the paper analyzes three distinct
aspects separately: the quality of local approximation of the embedding cloud,
the fitted-geometry fidelity of generated points, and their subsequent
discriminative utility in a downstream task.

%%%%%%%%%%%%%%%%%%%%%%%%%%%%%%%%
\section{Experimental study design}
\label{sec:experimental_design}

The experimental study in this work is organized around three clearly formulated
research questions that directly correspond to the main hypotheses of the paper.
This structure follows the logic of geometry-aware generation: first, we evaluate
the adequacy of the local geometric model; next, we assess the geometric validity
of the generated points; and only then do we examine the relation between their
geometric validity and their discriminative utility in a downstream task.

\subsection{Research questions}

The experiments are designed to answer three main questions.

\begin{itemize}
    \item \textbf{RQ1.} Do controlled local classes of semantically close sentences form embedding clouds with measurable nonlinear geometry?
    \item \textbf{RQ2.} Can fitted local surfaces be used to generate synthetic latent points that preserve the fitted geometry of the original embedding cloud?
    \item \textbf{RQ3.} How do geometrically valid synthetic latent points affect downstream classification, and does their geometric validity coincide with discriminative utility?
\end{itemize}

%===================================================
\subsection{Data and text representation}

The experimental study in this work is conducted on a \emph{controlled generated
dataset of semantically close sentences}. This choice is essential.

Existing paraphrase corpora, such as PIT-2015, MRPC, or QQP, are not suitable for
testing the local geometric hypothesis considered in this paper. First, they do
not contain large controlled sets of variants of the same sentence template with
a known mechanism of variation. Second, fitted quadric models, and especially
fitted cubic models, require local embedding clouds whose size substantially
exceeds the number of model parameters. Already for a quadric in a reduced space
of dimension $r=29$, the number of coefficients is
$
465,
$
whereas for a cubic model it is
$
4960.
$
In the raw embedding space of dimension $d=768$, the situation is even more
extreme: the full cubic model contains
$$
p_3(768)=\binom{768+3}{3}=76{,}088{,}705
$$
coefficients. Thus, even one observed point per coefficient would require more
than $76$ million examples, while statistically meaningful estimation of such a
number of parameters would require hundreds of millions of samples. Third, in
natural corpora the source of semantic variation is not transparent, which makes
it impossible to separate the local geometry of a class from the accidental
heterogeneity of the sample. For this reason, the use of a controlled
template-based dataset in this work is a necessary condition for testing the
proposed hypothesis, rather than a workaround.

To reduce the risk of an artifact caused by a single sentence template, the work
uses three different \emph{template families}, which differ in semantic field,
syntactic slot content, and the parts of speech that are varied.

\paragraph{Template family A (institutional communication and policy).}
$$
\texttt{The minister \{\} that the \{\} would bring \{\} for \{\}.}
$$
This template varies four slots: a reporting verb, a noun denoting a policy or
initiative, a noun denoting benefits or consequences, and a noun denoting a group
of beneficiaries.

\paragraph{Template family B (educational context).}
\[
\texttt{The teacher designed a \{\} lesson with \{\} exercises for \{\} students in a \{\} course.}
\]
This template varies four slots: the characteristic of the lesson, the
characteristic of the exercises, the characteristic of the student group, and the
characteristic of the course. In contrast to template A, this case varies not only
verbs and nouns, but also adjectival or attributive modifiers.

\paragraph{Template family C (medical context).}
\[
\texttt{The doctor recommended a \{\} treatment with \{\} monitoring for \{\} patients} \\ \texttt{during \{\} recovery.}
\]
This template varies four slots: the characteristic of the treatment, the
characteristic of monitoring, the characteristic of the patient group, and the
characteristic of the recovery phase.

For each of the four slots in each template family, $18$ context-compatible and
semantically close lexical variants were prepared. The full lists of variants are
given in Appendix~\ref{app:lexical_variants}. In this work, we deliberately avoid
the stronger claim that all these lexical items are strict synonyms in the
lexicographic sense; what matters is that they define controlled local
variability without radically changing the overall semantic profile of the
sentence.

Within each template family, the same sequence of experimental regimes
$C1$--$C5$ is constructed:
\begin{gather*}
\begin{aligned}
C1 &: \text{only slot } s_1 \text{ varies}, \quad |C1| = 18,\\
C2 &: \text{only slot } s_2 \text{ varies}, \quad |C2| = 18,\\
C3 &: \text{slots } s_1 \text{ and } s_2 \text{ vary jointly}, \quad |C3| = 18^2 = 324,\\
C4 &: \text{slots } s_1,s_2,s_3 \text{ vary jointly}, \quad |C4| = 18^3 = 5{,}832,\\
C5 &: \text{slots } s_1,s_2,s_3,s_4 \text{ vary jointly}, \quad |C5| = 18^4 = 104{,}976.
\end{aligned}
\end{gather*}
When it is necessary to distinguish template families explicitly in the analysis,
we use the notation $A\!\!-\!C1,\dots,A\!\!-\!C5$,
$B\!\!-\!C1,\dots,B\!\!-\!C5$, and
$C\!\!-\!C1,\dots,C\!\!-\!C5$.

For each template family, the full four-slot combinatorial space contains
$
18^4 = 104{,}976
$
unique sentences, and across the three template families the largest complete
set of $C5$ regimes contains
$
3\cdot 18^4 = 314{,}928
$
unique sentences. For regimes $C1$--$C4$, the non-varying slots are fixed at
preselected anchor values listed in Appendix~\ref{app:lexical_variants}. This
makes each regime fully reproducible.

All sentences are encoded using the same sentence embedding model. This choice is
deliberate, as it avoids mixing two different factors: the geometry of the data
itself and differences between embedding-model architectures. As a result, the
obtained results can be interpreted as properties of the local geometry of
controlled semantic classes in a fixed embedding space.

The main object of analysis is the full sentence embedding space. For
visualization and illustrative geometric figures, a three-dimensional
dimensionality reduction by $\mathrm{PCA}(3)$ is used. This three-dimensional
representation, however, is not the space in which quadrics and cubics are fitted
directly.

Direct fitted modeling is performed in a reduced space whose number of principal
components preserves at least $90\%$ of the total variance of the corresponding
local embedding cloud. Thus, the dimensionality of the reduced space is selected
adaptively for each class, rather than being fixed in advance at three. In some
experiments, this led to the use of a substantially larger number of coordinates,
up to \(r=37\) principal components. This choice is essential: it substantially reduces the
instability of approximation in the raw high-dimensional embedding space while
avoiding the misleading impression that the fitted quadric or fitted cubic model
is constructed only on a coarse three-dimensional projection.

\begin{table}[ht]
\centering
\caption{Adaptive PCA dimensionality selected by the \(90\%\) explained-variance criterion for all template families and regimes.}
\label{tab:pca_dimensions}
\begin{tabular}{lrrrr}
\hline
\textbf{Regime} & \textbf{\(N\)} & \textbf{Original dim.} & \textbf{\(r\)} & \textbf{Explained variance} \\
\hline
A-C1 & 18     & 768 & 10 & 0.922544 \\
A-C2 & 18     & 768 & 12 & 0.901424 \\
A-C3 & 324    & 768 & 16 & 0.906127 \\
A-C4 & 5,832  & 768 & 25 & 0.902540 \\
A-C5 & 104,976 & 768 & 37 & 0.902179 \\
\hline
B-C1 & 18     & 768 & 12 & 0.909688 \\
B-C2 & 18     & 768 & 11 & 0.920178 \\
B-C3 & 324    & 768 & 22 & 0.910898 \\
B-C4 & 5,832  & 768 & 31 & 0.905914 \\
B-C5 & 104,976 & 768 & 37 & 0.900155 \\
\hline
C-C1 & 18     & 768 & 12 & 0.903941 \\
C-C2 & 18     & 768 & 10 & 0.916903 \\
C-C3 & 324    & 768 & 20 & 0.902768 \\
C-C4 & 5,832  & 768 & 29 & 0.902631 \\
C-C5 & 104,976 & 768 & 36 & 0.903614 \\
\hline
\end{tabular}
\end{table}

Table~\ref{tab:pca_dimensions} shows that the effective dimensionality of the
local embedding clouds increases with the complexity of the slot-variation
regime. One-slot regimes \(C1\) and \(C2\) require only \(10\)--\(12\) principal
components to reach the \(90\%\) variance threshold, whereas the full four-slot
regimes \(C5\) require \(36\)--\(37\) components. This pattern is consistent with
the construction of the dataset: as more slots vary jointly, the number of active
semantic degrees of freedom increases, and the corresponding embedding cloud
occupies a higher-dimensional local region of the representation space.

Thus, the input data in this work have a fully specified origin, a known
combinatorial structure, controlled local semantic variability, and a reproducible
construction mechanism. This makes them suitable for directly testing our main
hypothesis: local classes of semantically close sentences in embedding space have
a regular geometric structure, rather than an arbitrary one, and this structure
can be explicitly modeled and used for latent probing.

An additional outcome of the work is the dataset \textbf{CoPaGE-300K}
(\emph{Controlled Paraphrase Geometry Embeddings})~\cite{CoPaGE300K}, a controlled
template-based dataset of semantically close sentences with slot-level annotation
and precomputed sentence embeddings. The dataset contains more than
$3\cdot 10^5$ sentences across several template families and complexity regimes.
In contrast to standard paraphrase datasets, this dataset explicitly records the
sources of semantic variability through slot values, making it suitable for
studying the local geometry of embedding space, the manifold-like structure of
semantically close sentences, and representation-level generation methods. The
dataset should not be treated as a natural paraphrase corpus; its value lies in
its controlled design, full factorial structure, and support for reproducible
geometric analysis. All sentence embeddings in this work were computed using the
\texttt{sentence-transformers/all-mpnet-base-v2} model, which produces
$768$-dimensional vector representations of sentences. As an auxiliary cross-model robustness check, we also repeated the main
affine-versus-quadric geometric comparison with
\texttt{sentence-transformers/all-MiniLM-L6-v2}, a smaller sentence-transformer
model that produces \(384\)-dimensional embeddings. These additional results are
reported in Appendix~\ref{app:minilm_robustness} and are used only to assess
whether the main geometric pattern persists beyond the primary embedding model.

\subsection{Compared approaches}

To evaluate the contribution of explicit local geometric modeling, the paper
compares several baseline ways of constructing points in embedding space. These
baselines represent the simplest scenarios of working in latent space and do not
use a fitted local carrier.

First, we consider the scenario with \emph{no added synthetic points}. This
serves as the baseline regime for downstream experiments and makes it possible
to check whether adding new points changes the behavior of the classifier.

Second, we use \emph{linear interpolation} between existing embeddings. This
baseline models the simplest interpolation-based way of constructing new latent
points and makes it possible to separate the effect of simple point mixing from
the effect of explicit alignment with fitted local geometry.

Third, we consider \emph{local random perturbation} in the reduced PCA space.
This baseline models a local ``thickening'' of the embedding cloud without
explicitly taking its fitted geometry into account. It is needed to test whether
simple local noise-based expansion is sufficient, or whether construction of a
fitted carrier is necessary.

Finally, we compare these baselines with \emph{surface-based generation}, where
the local embedding cloud is approximated by a low-degree fitted surface and
synthetic latent points are constructed with respect to this local geometric
structure. This variant directly tests the contribution of fitted-geometry-aware
generation relative to linear mixing and local noise-based perturbation.

\subsection{Evaluation metrics}
\label{subsec:metrics}

Since this work distinguishes between the description of local structure in
embedding space and the downstream use of generated points, two groups of metrics
are used: metrics of local approximation quality and fitted-geometry fidelity,
and metrics of discriminative utility in the classification task.

\paragraph{Geometric metrics.}
To analyze the quality of approximation of the local cloud, we use three basic
quantities: RMSE, mean absolute deviation, and mean algebraic fitting residual.
All three are computed in the same reduced space in which the fitted model is
actually constructed. Let $z_1,\dots,z_N \in \mathbb{R}^r$ be the points of the
local embedding cloud after dimensionality reduction, and let $f(z)=0$ be the
fitted local model. Then the algebraic residual for the point $z_i$ is defined as
$f(z_i)$. For a normalized model, the mean algebraic residual is computed as
$$
\frac{1}{N}\sum_{i=1}^{N} |f(z_i)|.
$$
If the normalized algebraic residual of each point with respect to the fitted
implicit model is additionally computed, then RMSE and mean absolute deviation are defined in the standard way.

In-sample fitting errors are interpreted only as a diagnostic of how well a given
low-degree model can reproduce the observed local cloud. They are not used as the
sole evidence of geometric regularity or generalization. For this reason, the
interpretation of the local geometry experiment also relies on held-out
validation results, which test whether the fitted local structure remains stable
beyond the points used for fitting.

To assess the validity of the generated points, five groups of metrics are used.

\emph{Surface consistency (SC)} evaluates how well synthetic points satisfy the
fitted-surface constraint. Let \(z_1^\ast,\dots,z_M^\ast\) be the generated
points. Then
$$
\mathrm{SC}=\frac{1}{M}\sum_{j=1}^{M} \tilde d_j,
$$
where \(\tilde d_j\) is the normalized algebraic surface residual of the point
\(z_j^\ast\) with respect to the fitted implicit model. The normalization is performed with respect to a
characteristic scale of the local cloud, for example the average nearest-neighbor
distance or the average radius of the cloud in the reduced space.

For surface-based generation, this metric should be interpreted as a
constraint-satisfaction measure rather than as an independent validation of the
method, since the generated points are explicitly projected onto the fitted
surface. The more informative tests of geometric stability are Hessian-based shape consistency and coefficient consistency, which evaluate how the fitted local
model changes after synthetic points have been added.

\emph{Neighborhood consistency (NC)} evaluates how naturally synthetic points
embed into the local neighborhood structure of the original embedding cloud. For
each generated point $z_j^\ast$, we first find the set of its $k$ nearest
neighbors among the original points:
$$
\mathcal{N}_k(z_j^\ast)\subset Z_{\mathrm{orig}}.
$$
In the experiments, we used \(k=5\) for the neighborhood-based metrics. This
choice keeps the neighborhood scale local even for relatively small controlled
classes and avoids comparing neighborhoods that cover a large fraction of the
cloud. Neighborhood consistency is therefore used as a local comparative
diagnostic between generation methods.

Next, for the point $z_j^\ast$, we determine its nearest original point
$
z_{i(j)}.
$
We then compare the $k$-neighborhood of the synthetic point $z_j^\ast$ with the
$k$-neighborhood of its nearest real analogue $z_{i(j)}$:
$$
\mathcal{N}_k(z_{i(j)})\subset Z_{\mathrm{orig}}.
$$
The normalized measure of local neighborhood agreement is defined as the fraction
of shared neighbors:
$$
\mathrm{Overlap}_k(z_j^\ast)
=
\frac{
\left|
\mathcal{N}_k(z_j^\ast)
\cap
\mathcal{N}_k(z_{i(j)})
\right|
}{k}.
$$
The overall neighborhood consistency is then defined as the mean value of this
quantity over all $M$ synthetic points:
$$
\mathrm{NC}
=
\frac{1}{M}
\sum_{j=1}^{M}
\mathrm{Overlap}_k(z_j^\ast).
$$
Thus, $\mathrm{NC}$ is a normalized dimensionless quantity in the interval
$[0,1]$. Larger values mean that the synthetic points better preserve the local
nearest-neighbor relations of the original cloud.

\emph{Distribution consistency (DC)} measures how well synthetic points agree
with the empirical local distribution of the original cloud. In the practical
implementation, it is computed as the mean normalized Mahalanobis distance:
$$
\mathrm{DC}=\frac{1}{M}\sum_{j=1}^{M}
\sqrt{(z_j^\ast-\mu)^\top \Sigma^{-1}(z_j^\ast-\mu)},
$$
where $\mu$ and $\Sigma$ are the empirical mean and covariance matrix of the
local cloud in the reduced space. Smaller values of DC indicate better agreement
of the synthetic points with the local distribution.

\medskip 
\emph{Hessian-based shape consistency} characterizes how stable the
Hessian-based second-order descriptors of the fitted implicit model remain after
synthetic points have been added. Since the fitted carriers are implicit
polynomial hypersurfaces in the reduced PCA space, we do not use the classical
Gaussian and mean curvature formulas for two-dimensional surfaces in
\(\mathbb{R}^3\). Instead, we use Hessian-based local shape descriptors that are
defined for the implicit polynomial model in arbitrary reduced dimension.

Let
$$
f_{\mathrm{orig}}(z)=0
\quad \text{and} \quad
f_{\mathrm{ext}}(z)=0
$$
denote the fitted implicit models before and after adding synthetic points,
respectively. Here \(f_{\mathrm{orig}}\) is fitted on the original embedding
cloud, whereas \(f_{\mathrm{ext}}\) is fitted on the extended set containing both
the original and the generated points.

It is important to clarify that shape-related quantities are not assigned
directly to the finite embedding cloud. A finite point cloud does not have
classical differential curvature by itself. The descriptors used here are
computed for the smooth fitted local carrier that approximates the cloud. Thus,
the metric evaluates the stability of the fitted model's local second-order
behavior, rather than an intrinsic curvature of the raw point set.

Next, we choose a set of representative points
$$
p_1,\dots,p_L,
$$
from the local region of the original embedding cloud. In the implementation,
these points are sampled from the original cloud in the reduced PCA space, rather
than from arbitrary locations on the fitted hypersurface. Points at which the
gradient norm of the fitted implicit model is numerically too small are excluded
from this diagnostic, since the descriptors below are normalized by
\(\|\nabla f(p)\|+\varepsilon\). Thus, the Hessian-based shape consistency is
computed away from numerically singular or nearly singular locations of the
implicit model. For a fitted model \(f\), two local shape descriptors are used:
$$
d_F(p)
=
\frac{\|\nabla^2 f(p)\|_F}
{\|\nabla f(p)\|+\varepsilon},
\qquad
d_2(p)
=
\frac{\|\nabla^2 f(p)\|_2}
{\|\nabla f(p)\|+\varepsilon},
$$
where \(\|\cdot\|_F\) is the Frobenius norm, \(\|\cdot\|_2\) is the spectral
norm, and \(\varepsilon>0\) is a small stabilizing constant.

The descriptor \(d_F\) measures the overall magnitude of second-order variation
of the fitted implicit model, whereas \(d_2\) measures the largest local
second-order direction. These quantities are norm-based descriptors: they do not
encode the sign pattern of the Hessian and therefore do not distinguish, for
example, saddle-like and locally convex behavior with the same Hessian norm.
They should be interpreted as numerical diagnostics of second-order shape
magnitude, not as full differential-geometric curvature invariants. The normalization by $\|\nabla f(p)\|+\varepsilon$ reduces dependence on the
arbitrary scale of the implicit equation and improves numerical stability.
Nevertheless, the descriptors are not evaluated at points where the gradient norm
is numerically close to zero; such points are treated as unstable for this
diagnostic.

For each representative point \(p_\ell\), the descriptors are computed for both
\(f_{\mathrm{orig}}\) and \(f_{\mathrm{ext}}\):
$$
d_F^{\mathrm{orig}}(p_\ell),\quad
d_2^{\mathrm{orig}}(p_\ell),
\qquad
d_F^{\mathrm{ext}}(p_\ell),\quad
d_2^{\mathrm{ext}}(p_\ell).
$$
The relative drift of the Frobenius-based descriptor is defined as
$$
D_F
=
\frac{1}{L}
\sum_{\ell=1}^{L}
\frac{
\left|
d_F^{\mathrm{orig}}(p_\ell)
-
d_F^{\mathrm{ext}}(p_\ell)
\right|
}{
\left|
d_F^{\mathrm{orig}}(p_\ell)
\right|+\varepsilon
},
$$
and the relative drift of the spectral-norm descriptor is defined as
$$
D_2
=
\frac{1}{L}
\sum_{\ell=1}^{L}
\frac{
\left|
d_2^{\mathrm{orig}}(p_\ell)
-
d_2^{\mathrm{ext}}(p_\ell)
\right|
}{
\left|
d_2^{\mathrm{orig}}(p_\ell)
\right|+\varepsilon
}.
$$
Finally, Hessian-based shape consistency is defined as
$$
\mathrm{ShapeCons}
=
\frac{1}{2}\left(D_F+D_2\right).
$$
The value \(\mathrm{ShapeCons}\) is dimensionless. Smaller values indicate that
adding synthetic points does not substantially change the Hessian-based
second-order descriptors of the fitted implicit carrier. This metric should be
interpreted as an aggregate stability diagnostic for the fitted local model. It does not measure full geometric congruence of the two fitted hypersurfaces and
does not provide a direct linguistic interpretation of these Hessian-based shape
descriptors.

\medskip
\emph{Coefficient consistency} evaluates the stability of the algebraic
description of the fitted local model after synthetic points have been added. To
this end, the coefficient vector of the model fitted on the original cloud is
compared with the coefficient vector of the same model refitted on the extended
set. Since an algebraic model is defined only up to a nonzero scale factor, we
compare normalized rather than raw coefficient vectors:
$$
\hat{\theta}=\frac{\theta}{\|\theta\|_2}.
$$
Taking into account that the vectors \(\theta\) and \(-\theta\) define the same
algebraic surface, \emph{Coefficient consistency} is defined as
$$
\mathrm{CoeffCons}
=
\min\Bigl(
\|\hat{\theta}_{\mathrm{orig}}-\hat{\theta}_{\mathrm{ext}}\|_2,\;
\|\hat{\theta}_{\mathrm{orig}}+\hat{\theta}_{\mathrm{ext}}\|_2
\Bigr),
$$
where \(\hat{\theta}_{\mathrm{orig}}\) and \(\hat{\theta}_{\mathrm{ext}}\) are
the normalized coefficient vectors of the fitted model before and after adding
synthetic points. The value \(\mathrm{CoeffCons}\) is dimensionless; smaller
values indicate better preservation of the algebraic description of the fitted
local model on the extended set.

\medskip
\subsection{Experiment 1: local geometry analysis}

The first experiment is designed to test the hypothesis that the local embedding
cloud of a controlled semantic class has a nonlinear structure and is better
described by a low-degree surface than by a linear affine subspace.

For each local semantic class, three models are constructed: a linear affine
model, a quadric, and a cubic model. Approximation errors are then computed for
each model. The main comparison is made between the linear approximation and the
nonlinear surfaces. The quadratic model is treated as the main nonlinear model
for comparison, while the cubic model is used as an additional test for cases of
more complex local second-order shape. Thus, the key point is to test whether a quadric is
sufficient as the basic form of a local nonlinear carrier.

It is expected that, at least for some classes, the quadratic or cubic
approximation will be substantially more accurate than the linear one.

\subsection{Experiment 2: geometric validity of generated points}

The second experiment tests whether the generated latent points are 
geometrically admissible elements of the same local class as the original 
embedding cloud. Validity is evaluated not with respect to a predefined 
``true manifold'', but with respect to the local structure recovered from the 
data: the fitted local model, the neighborhood structure, and the empirical 
distribution.

For each local class, new points are generated in three ways: by linear 
interpolation, by local random perturbation, and by the proposed surface-based 
method.  In the main geometric-validity experiment, each method generated \(M=3000\)
synthetic latent points for each evaluated full-regime local class. To check
that the conclusions were not an artifact of this particular value, an additional
robustness analysis was performed for
\(n_{\mathrm{synth}}\in\{1000,3000,5000,10000\}\). The five geometric validity tests described in 
Subsection~\ref{subsec:metrics} are then applied to each set of synthetic 
points: surface consistency, neighborhood consistency, distribution consistency, 
Hessian-based shape consistency, and coefficient consistency.

These tests examine five different aspects of consistency. Surface consistency 
measures the closeness of generated points to the fitted surface. Neighborhood 
consistency tests whether the local nearest-neighbor structure is preserved. 
Distribution consistency shows whether the new points become statistical outliers 
with respect to the local distribution. Hessian-based shape consistency measures whether the Hessian-based
second-order descriptors of the fitted implicit carrier change after adding
synthetic points. Coefficient consistency evaluates the stability of the algebraic 
description of the fitted model after augmentation.

Surface-based generation is expected to produce points that preserve the local 
fitted geometry better than simple random perturbation and, at least in some 
cases, are geometrically more stable than points obtained by linear 
interpolation.

\subsection{Experiment 3: geometric validity and discriminative utility}

The third experiment tests whether synthetic latent points constructed by different methods have discriminative utility in downstream classification under a small training-data regime. The first two experiments focus on the local geometry of embedding clouds and the geometric validity of generated points. Here, we address a different question: whether such geometric validity translates into discriminative utility for a particular classifier.

This experiment does not assume that adding synthetic points to the train set automatically improves accuracy or Macro-F1. If the original embedding space already contains a sufficient discriminative signal, additional points may provide no benefit for constructing the decision boundary. In some cases, they may even shift the empirical distribution of the train set in a direction that does not improve classification. Therefore, the result of this experiment is interpreted as a diagnostic test of the relation between two different properties: the \emph{geometric validity} of synthetic points and their \emph{discriminative utility}.

The downstream model is \emph{logistic regression}. This choice reduces the influence of classification-architecture complexity and allows us to focus on whether generated points change class separability in representation space. Training is performed in a few-shot regime, that is, with a small number of real examples per class. For each regime, four scenarios are compared: no added synthetic points, linear interpolation, local random perturbation, and surface-based generation.

The initial classification of template families \(A\), \(B\), and \(C\) is used as a sanity check. If these families are separated almost perfectly even with a very small number of training examples, the task provides little information for assessing the effect of generated points. Therefore, we additionally consider a more difficult slot-based setting, in which classes correspond to the values of one varying slot and the remaining slots form the context. In the final version of this task, the target slot is \(s_4\); its \(18\) possible values define \(18\) classes, while the context is given by the triple of slots \((s_1,s_2,s_3)\).

To impose a stronger generalization requirement, we use a \emph{context-held-out split}: the train and test sets are constructed so that the contexts \((s_1,s_2,s_3)\) do not overlap. In this case, the classifier is evaluated on new combinations of context slots that were not available during training. This setting makes the task more difficult and provides a clearer test of whether generated points add useful information for generalization.

The aim of the experiment is not to demonstrate a necessary improvement in classification performance, but to test whether geometric validity and discriminative utility coincide. If surface-based points are well aligned with the fitted local surface but do not improve downstream classification, this does not refute their geometric validity. Such a result indicates that the practical use of generated points in classification tasks may require an additional selection mechanism, for example one that takes into account the decision boundary, the inter-class margin, model uncertainty, or the informational utility of a point.

\subsection{Minimal ablation}

To avoid overloading the paper with a large number of auxiliary implementation
variants, the ablation analysis in this work is deliberately limited to two
comparisons that directly concern the central methodological decisions of the
proposed approach. These comparisons have not only a technical but also a
conceptual role, since they make it possible to separate the genuinely geometric
contribution of the method from simple ways of constructing synthetic latent
points.

The first comparison concerns the choice of the local geometric model. In this
work, the quadratic model is treated as the basic working variant of nonlinear
approximation of a local embedding patch, whereas the cubic model is treated as a
more flexible but parametrically more complex generalization. The ablation by
surface type is therefore not intended to prove the superiority of the more
complex model as such, but to determine whether moving from a quadric to a cubic
surface gives a substantial and stable improvement in the accuracy of the local
description. If the cubic model improves the fit only slightly while
substantially increasing the number of parameters, then the quadric remains a
methodologically justified baseline model.

The second comparison concerns the role of projection onto the fitted local
surface. Here two generation variants are compared: \emph{(i)} barycentric
generation without subsequent geometric alignment and \emph{(ii)} barycentric
generation followed by projection onto the fitted surface. This comparison makes
it possible to separate the effect of simple internal mixing of existing
embedding points from the effect of explicit alignment with the local geometric
model. If the variant with projection has a smaller surface residual, smaller Hessian-based shape drift, and smaller coefficient drift, then the central contribution of
the method lies not only in creating new points, but specifically in bringing
them onto the fitted local carrier.

At the same time, downstream classification is not treated as a mandatory
success criterion for every ablation variant. Adding synthetic points to the
train set does not by itself guarantee an improvement in accuracy or Macro-F1.
Within the ablation analysis, it is therefore important to distinguish two
properties: \emph{geometric validity} and \emph{discriminative utility}. The
first characterizes the consistency of synthetic points with the local structure
of the embedding cloud; the second depends on the specific downstream task, the
classifier, data density, and the position of the decision boundary.

%%%%%%%%%%%%%%%%%%%%%%%%%%%%%%%%%%%%%%%%%

\section{Results of the experimental study and analysis}

%%%%%%%%%%%%%%%%%%%%%%%%%%%%%%%%%%%%%%%%

This section presents the results of the experimental evaluation of the proposed
approach to local geometric modeling and latent probing of controlled semantic
variation in sentence embedding space. The exposition follows the structure
defined in the experimental design section. We first analyze whether local
embedding clouds have nonlinear geometry and whether they can be approximated by
low-degree fitted carriers. We then examine the geometric validity of synthetic
latent points generated by different methods. Finally, we evaluate whether
fitted-geometry fidelity translates into discriminative utility in a low-data
downstream setting. The section concludes with an ablation analysis of the main
components of the proposed approach.

\subsection{Results of the local geometry study}

The first experiment tests whether local embedding clouds of controlled semantic
classes have a structure that goes beyond simple affine geometry. For each local
class, three types of models were compared: affine approximation, a quadratic
model, and a cubic model. Approximation quality was evaluated by the RMSE and MAE
of normalized residuals in the same reduced space in which the corresponding
fitted model was constructed.

All fitting errors reported in this experiment are computed in the adaptive
reduced PCA space selected separately for each regime. The corresponding
dimensions are reported in Table~\ref{tab:pca_dimensions}. This is important
because the fitted models are not constructed in the full \(768\)-dimensional
ambient space, nor in a fixed three-dimensional visualization space, but in the
effective local coordinate system of each embedding cloud.

As shown in Table~\ref{tab:fit_quality}, moving from the affine model to the
quadratic approximation noticeably reduces both RMSE and MAE for all
representative regimes reported here. The cubic model gives an additional
improvement in some cases, although this gain is not uniform across all classes.
This is consistent with the working hypothesis: a local embedding cloud has a
nonlinear structure, while the complexity of this structure depends on the
particular template family and the slot-variation regime.

\begin{table}[ht]
\centering
\caption{Comparison of local approximation quality of embedding clouds using different geometric models}
\label{tab:fit_quality}
\begin{tabular}{|c|c|c|c|}
\hline
\textbf{Regime} & \textbf{Model} & \textbf{RMSE} & \textbf{MAE} \\
\hline
A-C1 & Affine model   & 0.077470 & 0.060474 \\
A-C1 & Quadric        & 0.000000 & 0.000000 \\
A-C1 & Cubic model    & 0.000000 & 0.000000 \\
\hline
A-C3 & Affine model   & 0.186873 & 0.162540 \\
A-C3 & Quadric        & 0.005633 & 0.004195 \\
A-C3 & Cubic model    & 0.000000 & 0.000000 \\
\hline
A-C4 & Affine model   & 0.283923 & 0.232159 \\
A-C4 & Quadric        & 0.011931 & 0.009358 \\
A-C4 & Cubic model    & 0.004837 & 0.003649 \\
\hline
A-C5 & Affine model   & 0.229635 & 0.185434 \\
A-C5 & Quadric        & 0.016461 & 0.012982 \\
A-C5 & Cubic model    & 0.008195 & 0.006394 \\
\hline
B-C1 & Affine model   & 0.202334 & 0.144890 \\
B-C1 & Quadric        & 0.000000 & 0.000000 \\
B-C1 & Cubic model    & 0.000000 & 0.000000 \\
\hline
B-C3 & Affine model   & 0.184114 & 0.151228 \\
B-C3 & Quadric        & 0.002396 & 0.001300 \\
B-C3 & Cubic model    & 0.000000 & 0.000000 \\
\hline
B-C5 & Affine model   & 0.157735 & 0.125820 \\
B-C5 & Quadric        & 0.014219 & 0.010827 \\
B-C5 & Cubic model    & 0.007791 & 0.005873 \\
\hline
C-C1 & Affine model   & 0.190155 & 0.165462 \\
C-C1 & Quadric        & 0.000000 & 0.000000 \\
C-C1 & Cubic model    & 0.000000 & 0.000000 \\
\hline
C-C3 & Affine model   & 0.144995 & 0.115482 \\
C-C3 & Quadric        & 0.003477 & 0.002440 \\
C-C3 & Cubic model    & 0.000000 & 0.000000 \\
\hline
C-C5 & Affine model   & 0.117400 & 0.093637 \\
C-C5 & Quadric        & 0.016066 & 0.012460 \\
C-C5 & Cubic model    & 0.009763 & 0.007464 \\
\hline
\end{tabular}
\end{table}

The values \(0.000000\) in the table do not correspond to exact analytical
equality to zero, but to numerically very small errors that become zero after
rounding. Such values arise mainly for small or rigidly structured local classes,
where a low-degree model can almost exactly reproduce the observed cloud in the
corresponding reduced space. Therefore, the in-sample fit is not used as the only
argument in favor of a more complex model; approximation quality on held-out data
is also analyzed.

Therefore, the zero or near-zero in-sample values in small regimes should not be
read as independent evidence of a true geometric law. They indicate that the
chosen low-degree model is sufficiently flexible to reproduce the observed cloud
in the selected reduced space. The stronger evidence for non-affine local
structure is provided by the held-out validation comparison reported below.

\begin{figure}[ht]
\centering
 \includegraphics[width=0.89\textwidth]{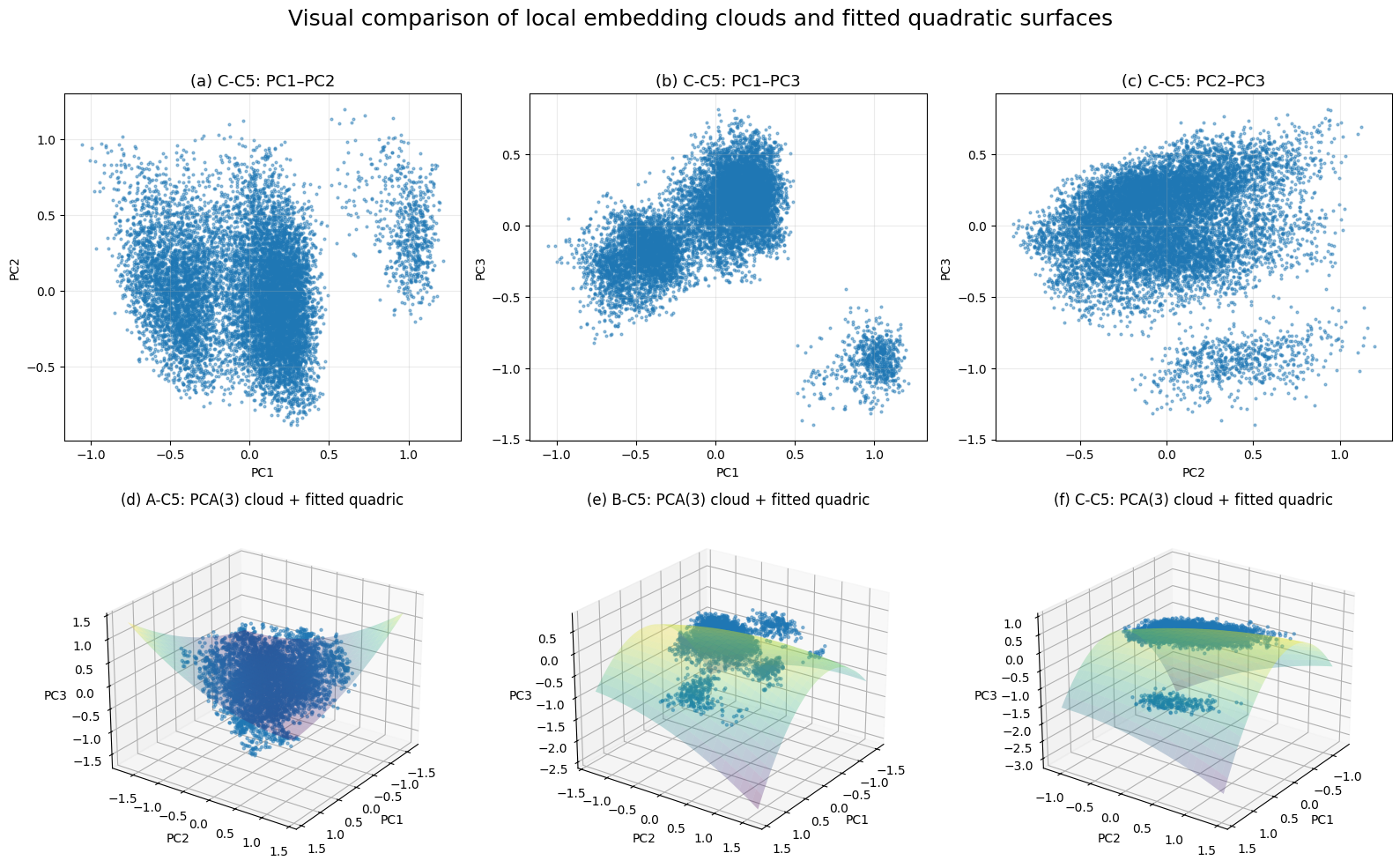}
\caption{Visual comparison of representative local embedding clouds and fitted quadratic surfaces.
Top row: pairwise PCA(3) projections of the C-C5 cloud, illustrating its clearly nonlinear structure. Bottom row: PCA(3) clouds for A-C5, B-C5, and C-C5 together with fitted quadratic surfaces shown for visualization purposes.}
\label{fig:local_cloud_quadric_generated}
\end{figure}

The visual analysis in Fig.~\ref{fig:local_cloud_quadric_generated} is consistent
with the quantitative results. Even in the three-dimensional PCA representation,
representative local clouds show a clear curvilinear structure, and the fitted
quadric reproduces their general shape for different template families. This
figure is illustrative: in the main experiments, fitted modeling is performed not
in PCA(3), but in the adaptive reduced space defined by the variance-preservation
criterion.

\begin{figure}[ht]
\centering
 \includegraphics[width=0.89\textwidth]{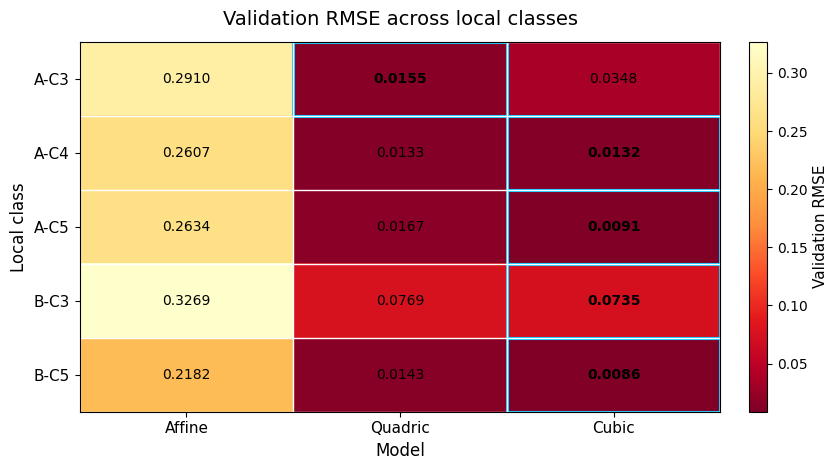}
\caption{Validation RMSE across representative local classes and geometric models.
Lower values indicate better approximation quality on held-out data. Affine models are consistently weaker, while quadratic and cubic models capture the local geometry substantially better.}
\label{fig:validation_rmse_heatmap}
\end{figure}

The quantitative comparison on held-out data is shown in
Fig.~\ref{fig:validation_rmse_heatmap}. The holdout analysis shows that the
advantage of nonlinear models is not reducible to simple overfitting to the
original embedding cloud. The affine model is systematically weaker than the
quadratic and cubic approximations, while the relation between the quadric and
the cubic depends on the particular class. In most of the considered regimes, the
cubic model gives the lowest validation RMSE values, but in some cases the
quadric already provides a strong baseline approximation.

Thus, the results of the first experiment confirm that the local geometry of
controlled semantic classes in embedding space generally cannot be reduced to a
simple affine structure. At the same time, they do not prove a universal
advantage of a single fixed model: the quadric serves as the baseline nonlinear
form, while the cubic surface is a useful extension for classes with more
complex local second-order shape.

%%%%%%%%%%%%%%%%%%%%%%%%%%%%%%%%%%%%%%%%%%%%%%%%%
\subsection{Results of geometric validity assessment of generated points}

The second experiment was aimed at testing whether new latent points generated by
different methods remain consistent with the local structure of the same class as
the original embedding cloud. To this end, three approaches to generating
synthetic points were compared: linear interpolation between existing points,
local random perturbation in the reduced PCA space, and the proposed
surface-based method.

It is important that the metrics used here do not describe the same aspect of
quality. Some of them characterize \emph{fitted-geometry fidelity}, that is, the
consistency of generated points with the fitted local carrier. Others characterize
\emph{empirical-density fidelity}, that is, the closeness of synthetic points to
the empirical density of the original cloud. These two goals are not identical
and may be in tension with each other.

The evaluation was performed using the five groups of criteria described in
Subsection~\ref{subsec:metrics}. In Table~\ref{tab:exp2_robustness_nsynth}, they
are represented by six numerical indicators, since local neighborhood consistency
is additionally characterized both by an overlap metric and by a deviation
measure. The \emph{Surface} metric evaluates the consistency of synthetic points
with the fitted local surface; \emph{Neighborhood} is the overlap of the local
nearest-neighbor structure; \emph{Neigh.\ dev.} is the normalized deviation of
the local neighborhood structure; \emph{Distr.\ dev.} is the deviation from the
empirical local distribution; Hessian-based shape consistency measures the change in Hessian-based
second-order descriptors of the fitted implicit carrier after adding synthetic
points; and \emph{Coefficient consistency} measures the change in normalized
coefficients of the fitted local model after expanding the cloud. For the
\emph{Neighborhood} metric, larger values are better, whereas for all other
metrics smaller values are better.

{\small 
\begin{table}[ht]
\centering
\caption{Comparison of synthetic latent point generation methods by two groups of criteria: fitted-geometry fidelity and empirical-density fidelity. For Surface, Neigh.\ dev., Distr.\ dev., Shape consistency, and Coefficient, smaller values are better; for Neighborhood, larger values are better.}
\label{tab:exp2_robustness_nsynth}
\begin{tabular}{llcccccc}
\hline
\(n_{\mathrm{synth}}\) & Method & Surface & Neighborhood & Neigh.\ dev. & Distr.\ dev. & Hess.\ shape & Coeff.\ cons.\\
\hline
1000 & Linear interpolation & 0.070755 & 0.739227 & 0.097070 & 0.209146 & 0.042890 & 0.213461 \\
1000 & Local perturbation   & 0.043045 & 0.898068 & 0.005455 & 0.060048 & 0.015380 & 0.066209 \\
1000 & Surface-based        & 0.000000 & 0.636248 & 0.469299 & 0.668771 & 0.000513 & 0.001547 \\
\hline
3000 & Linear interpolation & 0.068067 & 0.736525 & 0.111230 & 0.208537 & 0.076677 & 0.476835 \\
3000 & Local perturbation   & 0.043659 & 0.901314 & 0.008103 & 0.056265 & 0.032414 & 0.156697 \\
3000 & Surface-based        & 0.000000 & 0.639355 & 0.466354 & 0.657309 & 0.001019 & 0.002490 \\
\hline
5000 & Linear interpolation & 0.070725 & 0.733468 & 0.115292 & 0.202545 & 0.078814 & 0.604684 \\
5000 & Local perturbation   & 0.042220 & 0.900363 & 0.006882 & 0.060561 & 0.043751 & 0.214301 \\
5000 & Surface-based        & 0.000000 & 0.641130 & 0.462812 & 0.656208 & 0.001465 & 0.003314 \\
\hline
10000 & Linear interpolation & 0.072461 & 0.735955 & 0.122983 & 0.211337 & 0.552054 & 0.872295 \\
10000 & Local perturbation   & 0.043768 & 0.900970 & 0.007561 & 0.060564 & 0.097977 & 0.379360 \\
10000 & Surface-based        & 0.000000 & 0.639764 & 0.464961 & 0.657248 & 0.001711 & 0.004115 \\
\hline
\end{tabular}
\end{table}
}

The near-zero \emph{Surface} values for surface-based generation are expected,
because this method explicitly projects candidate points onto the fitted local
surface. They should therefore be interpreted as evidence that the projection
constraint is satisfied, rather than as an independent empirical advantage. The
more informative stability tests are Hessian-based shape consistency and coefficient
consistency, since they evaluate how the fitted local carrier changes after the
synthetic points have been added.

Table~\ref{tab:exp2_robustness_nsynth} does not show dominance of the
surface-based method across all criteria. It reveals a trade-off between two
different goals of generation. Local perturbation generates points directly in
small neighborhoods of existing embeddings; therefore, it better preserves the
empirical density of the cloud and obtains better values for the
\emph{Neighborhood}, \emph{Neigh.\ dev.}, and \emph{Distr.\ dev.} metrics. In
contrast, the surface-based method is optimized for preserving fitted local
geometry: consistency with the local surface,  second-order shape stability, and stability
of the fitted-model coefficients. This is why it obtains the best values for the
\emph{Surface}, \emph{Shape consistency}, and \emph{Coefficient} metrics.

This separation of results is not a numerical contradiction. Projection onto the
fitted surface can shift synthetic points relative to the nearest empirical
neighborhood regions, which explains the weaker neighborhood- and
distribution-based metrics. At the same time, this projection provides an almost
zero surface residual and a minimal change in the local analytic form of the
fitted carrier. Therefore, the surface-based method should not be interpreted as
a method that best reproduces the empirical density of the data; its advantage
lies in geometrically controlled preservation of the fitted local structure.

The additional comparison for different values of \(n_{\mathrm{synth}}\) confirms
that this separation is stable and is not an artifact of choosing \(3000\)
synthetic points. Thus, the second experiment shows not an absolute advantage of
one method across all metrics, but a difference between two types of validity. If
the goal is to reproduce the local empirical density as accurately as possible,
local perturbation is the natural choice. If the goal is to preserve the fitted geometry of the local manifold-like patch,
its Hessian-based second-order descriptors, and parameter stability, then surface-based generation has the methodological advantage.

Figure~\ref{fig:generated_points_comparison} presents a visual comparison of
synthetic points generated by the three methods for one representative local
class in the \(\mathrm{PCA}(3)\) space. This figure makes it possible to see
directly that linear interpolation forms points along internal chords of the
cloud, local random perturbation concentrates near empirical observations, and
the surface-based method generates points that are consistent with the fitted
surface. At the same time, the figure is illustrative: the quantitative metrics
were computed in the adaptive reduced space, not only in the three-dimensional
PCA projection.

\begin{figure}[ht]
\centering
\begin{center}
\includegraphics[width=0.89\textwidth]{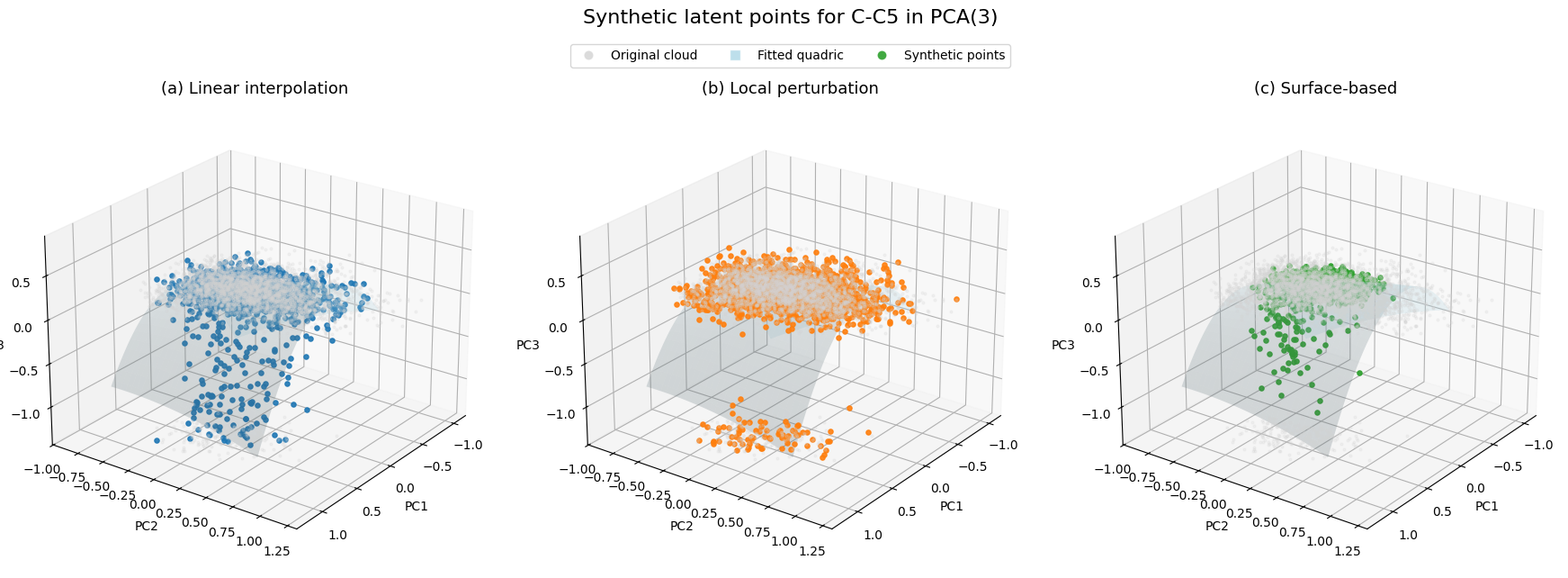}
\end{center}
\caption{Comparison of synthetic latent points generated by linear interpolation, local random perturbation, and the proposed surface-based method for one representative local class in the \(\mathrm{PCA}(3)\) space.}
\label{fig:generated_points_comparison}
\end{figure}

Thus, the second experiment shows that the geometric validity of generated latent
points should be understood as multidimensional. Surface-based generation best
preserves the fitted local geometry, whereas local perturbation more accurately
reproduces the empirical density of the original cloud. This result is not an
assessment of the downstream utility of synthetic points; it only establishes
which aspects of the local structure are preserved by each generation method.

%%%%%%%%%%%%%%%%%%%%%%%%%%%%%%%%%%%%
\subsection{Geometric validity and downstream discriminative utility}

The third experiment was aimed at examining the relation between two different properties of synthetic latent points: their geometric validity with respect to the fitted local structure and their discriminative utility in a downstream task. In this setting, downstream classification is used not as the sole criterion of success for the method, but as a diagnostic test of whether fitted-geometry fidelity translates into practical benefit for a classifier.

The initial task of classifying template families \(A\), \(B\), and \(C\) turned out to be too simple: a linear classifier separated these families almost perfectly even with a very small number of training examples. This means that the embedding space already contains a strong discriminative signal for this setting, so adding synthetic points does not provide a meaningful test of their effect. For this reason, we constructed a more difficult \emph{context-held-out slot classification task}.

This task was built on the full four-slot space of template family \(C\), that is, on the \(C\text{-}C5\) regime. The target slot was chosen to be \(s_4\), which has \(18\) possible values; each value of \(s_4\) was treated as a separate class. Thus, the downstream task was an \(18\)-class classification problem. The context for the target slot \(s_4\) was defined as the triple of the remaining slots
$$
(s_1,s_2,s_3).
$$
Since each of these slots has \(18\) values, the total number of possible contexts was
$
18^3 = 5832.
$
For each value of the target slot \(s_4\), there was one sentence in each such context. Therefore, each class contained \(5832\) examples, and the full \(C\text{-}C5\) set contained
$
18 \times 5832 = 104976
$
sentences.

The constructed lookup contained exactly one sentence for each pair of target
label and context, with no duplicate \((\text{label},\text{context})\) entries;
therefore, the task has the full factorial label-context structure required by
the context-held-out protocol.

The train/test split was constructed at the level of contexts, rather than at the level of individual sentences. For each random run, train contexts and test contexts were selected so that they did not overlap. If a context \((s_1,s_2,s_3)\) was assigned to the train set, then all sentences with this context, for all \(18\) values of \(s_4\), were used only for training. If a context was assigned to the test set, then all corresponding sentences were used only for testing. This ensured that the classifier was evaluated on new combinations of contextual slots that it had not seen during training.

In the \(k\)-shot regime, \(k\) train contexts were used. Thus, the train set contained
$
k \times 18
$
real training examples, that is, \(k\) examples for each of the \(18\) classes. For testing, \(300\) test contexts were used, giving
$
5400
$
test examples. For each train split, the results were averaged over \(10\) random runs. Synthetic points were added only to the train set. The test set was not used for
PCA construction, fitted-surface estimation, or synthetic point generation.

For each \(k\)-shot train split, the number of synthetic points added by each
generation method was equal to the number of real training examples. Since the
train set contains \(18k\) real examples, each method added \(M=18k\) synthetic
points, corresponding to a synthetic-to-real ratio of \(1:1\). Thus,
\(M=18,36,54,90,180\) for \(k=1,2,3,5,10\), respectively.

For very small \(k\), the surface-based variant may be underdetermined for some
classes because too few real points are available to fit a stable quadratic
carrier in the train-only reduced space. In such cases, the implementation uses
a fallback internal interpolation for that class. This is one reason why the
downstream experiment is interpreted diagnostically rather than as a direct test
of the full surface-based construction under all few-shot conditions.

The results of context-held-out slot classification for different few-shot regimes are shown in Table~\ref{tab:exp3_downstream_context_heldout}.

\begin{table}[ht]
\centering
\caption{Results of context-held-out slot classification for \(C\text{-}C5\), target slot \(s_4\). Values are reported as mean \(\pm\) standard deviation over 10 random runs.}
\label{tab:exp3_downstream_context_heldout}
\small
\begin{tabular}{clcc}
\hline
\textbf{Train size} & \textbf{Method} & \textbf{Accuracy} & \textbf{Macro-F1} \\
\hline
1 & No synthetic points & \(0.4439 \pm 0.0883\) & \(0.4179 \pm 0.0808\) \\
1 & Linear interpolation & \(0.4213 \pm 0.1012\) & \(0.4046 \pm 0.0967\) \\
1 & Local random perturbation & \(0.4212 \pm 0.1014\) & \(0.4045 \pm 0.0969\) \\
1 & Surface-based generation & \(0.4213 \pm 0.1012\) & \(0.4046 \pm 0.0967\) \\
\hline
2 & No synthetic points & \(0.3742 \pm 0.0369\) & \(0.3334 \pm 0.0520\) \\
2 & Linear interpolation & \(0.3640 \pm 0.0354\) & \(0.3346 \pm 0.0472\) \\
2 & Local random perturbation & \(0.3355 \pm 0.0380\) & \(0.3046 \pm 0.0458\) \\
2 & Surface-based generation & \(0.3654 \pm 0.0310\) & \(0.3315 \pm 0.0461\) \\
\hline
3 & No synthetic points & \(0.3712 \pm 0.0706\) & \(0.3237 \pm 0.0829\) \\
3 & Linear interpolation & \(0.3662 \pm 0.0637\) & \(0.3310 \pm 0.0761\) \\
3 & Local random perturbation & \(0.3456 \pm 0.0688\) & \(0.3070 \pm 0.0731\) \\
3 & Surface-based generation & \(0.3514 \pm 0.0752\) & \(0.3127 \pm 0.0899\) \\
\hline
5 & No synthetic points & \(0.3390 \pm 0.0431\) & \(0.2908 \pm 0.0543\) \\
5 & Linear interpolation & \(0.3379 \pm 0.0423\) & \(0.3018 \pm 0.0520\) \\
5 & Local random perturbation & \(0.3128 \pm 0.0441\) & \(0.2810 \pm 0.0455\) \\
5 & Surface-based generation & \(0.3210 \pm 0.0466\) & \(0.2793 \pm 0.0556\) \\
\hline
10 & No synthetic points & \(0.3565 \pm 0.0451\) & \(0.3163 \pm 0.0572\) \\
10 & Linear interpolation & \(0.3500 \pm 0.0414\) & \(0.3160 \pm 0.0539\) \\
10 & Local random perturbation & \(0.3407 \pm 0.0474\) & \(0.3100 \pm 0.0589\) \\
10 & Surface-based generation & \(0.3525 \pm 0.0376\) & \(0.3173 \pm 0.0486\) \\
\hline
\end{tabular}
\end{table}

The identical results of linear interpolation and surface-based generation in
the \(1\)-shot regime are explained by the fallback behavior of the
surface-based procedure. With only one real example per class, a stable local
surface cannot be fitted in the train-only reduced space. In this case, the
surface-based variant reduces to an internal interpolation-type construction.
This affects only the most extreme few-shot setting and is one of the reasons why
the downstream experiment is interpreted diagnostically rather than as a direct
test of the full surface-based construction under all data regimes.

The results do not show a stable downstream gain from adding synthetic latent points. In most few-shot regimes, the baseline with no synthetic points remains the best in terms of accuracy or close to the best. Linear interpolation and surface-based generation behave more stably than local random perturbation, but they do not provide a systematic improvement over the no-synthetic-points baseline. Local random perturbation most often gives the weakest results, indicating the risk of uncontrolled noise-based expansion of the train set.

The observed differences between methods are small relative to the variability
across random context splits. Therefore, the downstream experiment should be
interpreted as a diagnostic test rather than as evidence of a statistically
robust performance advantage of any generation method. Its main role is to show
that fitted-geometry fidelity does not automatically imply discriminative
utility.

The non-monotonic behavior of the results with respect to \(k\) reflects the difficulty of the context-held-out setting: performance depends not only on the number of train contexts, but also on which contextual combinations are available during training. Therefore, increasing the number of few-shot examples does not necessarily lead to a monotonic improvement in accuracy or Macro-F1 for every set of random splits.

Importantly, this result does not contradict the conclusions of the first two experiments. Surface-based generation shows good fitted-geometry fidelity: the generated points are consistent with the fitted surface, stabilize local second-order shape, and better preserve the coefficients of the fitted model. However, these properties do not guarantee that the points add new information for constructing the decision boundary. In other words, a geometrically valid point may be correct with respect to a local manifold-like patch, but discriminatively redundant for a particular classifier.

A similar logic is discussed in the literature on data augmentation for NLP. In particular, Longpre et al.~\cite{Longpre2020} explain the limited effectiveness of many simple augmentation methods for large pretrained language models by the fact that such models have already learned invariance to many surface-level or lexical transformations. Our setting does not involve fine-tuning a language model, but rather logistic regression on top of sentence embeddings; however, the general conclusion is methodologically close: if synthetic examples are too close to the original ones in representation space, they may fail to add new discriminative patterns and mostly duplicate information already present in the data.

Thus, the third experiment clarifies the limits of interpretation of the proposed approach. It shows that fitted-geometry fidelity and downstream discriminative utility are distinct properties. The proposed method should not be interpreted as a universal classification booster or as an automatic way to improve accuracy or Macro-F1. Its main role is geometrically controlled modeling of local semantic patches and construction of synthetic latent points that are valid with respect to the fitted local structure.

The practical consequence is that using surface-based generation in downstream classification requires an additional mechanism for selecting synthetic points. Such a mechanism should take into account not only whether a point lies on the fitted surface, but also its potential contribution to class separation, for example its distance to the decision boundary, classifier uncertainty, the inter-class margin, or the informational utility of the point in active learning. Without such discriminative filtering, geometrically correct latent generation may remain redundant for models that already have strong representation-level invariances.

%%%%%%%%%%%%%%%%%%%%%%%%%%%%%%%%%%%%%%%%%%%%%%%%%%%%%%%%%%%%%%
\subsection{Ablation analysis}

The ablation analysis in this work has a methodological rather than an engineering
purpose. Its aim is not to prove a guaranteed downstream-quality gain from each
generation variant, but to separate the contribution of individual components of
the proposed scheme: the type of fitted local model, projection onto the fitted
surface, and the addition of synthetic points to the train set. This is consistent
with the general interpretation of the work, where geometric validity and
discriminative utility are treated as distinct properties.

The first ablation concerns the choice of the local geometric model. The
comparison of affine, quadratic, and cubic approximations shows that the affine
model does not adequately describe the local structure of embedding clouds. The
quadratic model substantially reduces the approximation error and is therefore
treated as the baseline working nonlinear model. The cubic model provides an
additional fit improvement in some cases, but has a substantially larger number
of parameters. Thus, it should be interpreted as an optional local extension for
more complex classes, rather than as a universal replacement for the quadric.

The second ablation concerns the role of projection onto the fitted local surface.
The results of the second experiment show that surface-based generation provides
an almost zero surface residual and the lowest values of the metrics associated
with fitted-geometry fidelity, in particular Hessian-based shape consistency and
coefficient consistency. This means that projection plays not only a technical
role, but also a methodological one: it aligns synthetic points with the local
fitted carrier and distinguishes the proposed approach from simple internal
mixing or noise-based perturbation of points.

At the same time, this geometric advantage does not imply an automatic downstream
gain. Therefore, in the ablation analysis, classification results are reported as
deviations from the baseline with no synthetic points. This makes it possible to
directly assess whether each type of synthetic points adds discriminative utility
compared with training only on real examples. The corresponding results for the
10-shot context-held-out slot task are shown in
Table~\ref{tab:ablation_downstream}. The columns \emph{Best Acc. count} and
\emph{Best F1 count} indicate in how many of the 10 random runs the corresponding
method achieved the best accuracy or Macro-F1.

\begin{table}[ht]
\centering
\caption{Ablation comparison of the downstream effect of generation methods relative to the baseline with no synthetic points for the 10-shot context-held-out slot task. \(\Delta\) denotes the difference relative to the baseline with no synthetic points.}
\label{tab:ablation_downstream}
\begin{tabular}{lcccc}
\hline
\textbf{Method} &
\(\Delta\)\textbf{Accuracy} &
\(\Delta\)\textbf{Macro-F1} &
\begin{tabular}[c]{@{}c@{}}\textbf{Best Acc.}\\ \textbf{count}\end{tabular} &
\begin{tabular}[c]{@{}c@{}}\textbf{Best F1}\\ \textbf{count}\end{tabular} \\
\hline
No synthetic points & 0.0000 & 0.0000 & 5/10 & 3/10 \\
Linear interpolation & -0.0065 & -0.0002 & 1/10 & 1/10 \\
Local random perturbation & -0.0158 & -0.0063 & 1/10 & 1/10 \\
Surface-based generation & -0.0040 & +0.0010 & 3/10 & 5/10 \\
\hline
\end{tabular}
\end{table}

As shown in Table~\ref{tab:ablation_downstream}, none of the scenarios with
synthetic points provides a stable accuracy improvement over the baseline. Local
random perturbation has the weakest average effect, indicating the risk of
uncontrolled noise-based expansion of the train set. Surface-based generation is
slightly below the baseline in accuracy, but remains the most competitive among
the scenarios with synthetic points and has the largest number of wins in terms
of Macro-F1. However, these differences are not sufficient to claim a systematic
downstream advantage.

Thus, the ablation analysis confirms the role of projection onto the fitted
surface primarily from the geometric perspective: it best preserves the local
analytic and local second-order structure of the embedding cloud. From the perspective of
downstream classification, this is not sufficient: geometrically correct points
may be redundant if they do not add new discriminative patterns or refine the
decision boundary. This is consistent with the main conclusion of the third
experiment, namely the distinction between fitted-geometry fidelity and
downstream discriminative utility.

Therefore, surface-based generation should be viewed primarily as an offline tool
for geometrically controlled latent-space generation and analysis of local
semantic patches, rather than as a universal mechanism for improving
classification. Practical use of generated points in downstream tasks requires
additional discriminative filtering or active selection, which takes into account
not only whether a point lies on the fitted surface, but also its potential
contribution to the construction of the decision boundary.

\subsection{Cross-model robustness check}

To assess whether the main geometric pattern is specific to the primary
embedding model, we performed an auxiliary cross-model robustness check with
\texttt{sentence-transformers/all-MiniLM-L6-v2}. This model produces
\(384\)-dimensional sentence embeddings, whereas the main experiments use
\texttt{sentence-transformers/all-mpnet-base-v2} with \(768\)-dimensional
embeddings. The purpose of this check is not to repeat the full experimental
pipeline, but to test whether the affine-versus-quadratic contrast persists in a
different sentence-transformer embedding space.

For \texttt{all-MiniLM-L6-v2}, the adaptive PCA analysis again showed that the
effective dimensionality of the controlled semantic clouds increases with the
complexity of the slot-variation regime. In the full four-slot regimes, the
\(90\%\) explained-variance criterion selected \(r=38\), \(r=39\), and \(r=38\)
for \(A\text{-}C5\), \(B\text{-}C5\), and \(C\text{-}C5\), respectively. The full
PCA summary for this auxiliary check is reported in
Appendix~\ref{app:minilm_robustness}.

For computational reasons, the affine-versus-quadric fitting check was then
performed in a fixed reduced space of dimension \(r=20\). This setting is
deliberately lighter than the full \(90\%\)-variance reduced spaces and is
therefore interpreted only as an auxiliary robustness check. For the MiniLM
\(C5\) regimes, this \(r=20\) space preserves \(0.6978\), \(0.6791\), and
\(0.7005\) of the total PCA variance for \(A\text{-}C5\), \(B\text{-}C5\), and
\(C\text{-}C5\), respectively. Thus, the auxiliary check is intentionally
performed in a compressed local coordinate space, and the results should be read
as a conservative test of the affine-versus-quadric contrast.

More specifically, for \(A\text{-}C5\), the surface RMSE decreased from
\(0.732287\) for the affine model to \(0.114082\) for the quadric model. For
\(B\text{-}C5\), it decreased from \(0.482266\) to \(0.082574\), and for
\(C\text{-}C5\), from \(0.550278\) to \(0.074483\). Thus, the nonlinear local
structure detected in the main experiments is not merely an artifact of the
original \(768\)-dimensional MPNet-based embedding space.

This auxiliary check has a limited scope: it does not replace a systematic
cross-model study over many embedding architectures, nor does it repeat the full
geometric-validity and downstream evaluation pipeline. Nevertheless, it supports
the qualitative robustness of the central geometric conclusion: controlled
semantic variation induces local embedding clouds that are better described by
nonlinear fitted carriers than by affine approximations.

\medskip
Taken together, the experiments show that controlled semantic variation induces measurable local geometry in sentence embedding space. This geometry can be described by low-degree fitted models and used to construct synthetic latent points consistent with the fitted local structure. At the same time, the downstream experiments show that fitted-geometry fidelity does not guarantee discriminative utility. This distinction provides the basis for the discussion that follows.

\section{Discussion}

The study confirms that controlled template-based classes of semantically close sentences exhibit local geometric regularity in sentence embedding space. This regularity is reflected in the fact that the corresponding embedding clouds are better described by nonlinear fitted models than by affine subspaces. The results support the working local manifold hypothesis: semantic variation within a narrow controlled class has a reproducible local structure.

The proposed approach is best understood as a tool for local geometric analysis of representation space. Its main purpose is to explicitly model the local geometry of controlled semantic patches in latent space and to examine which properties are preserved by synthetic latent points constructed with respect to a fitted local carrier. The computational cost of the method should therefore be associated with an offline procedure for geometric analysis, latent probing, and controlled latent-space generation.

PCA plays a stabilizing role in this work. Polynomial fitting and Newton-type projection are performed in a reduced PCA space of dimension \(r\), where the dominant variability of the corresponding local cloud is preserved. This step is consistent with nonlinear modeling, since the fitted model describes the effective local geometry of a particular semantic patch rather than the entire ambient embedding space.

The recorded PCA summaries confirm that this step behaves consistently with the
experimental design. The selected dimension grows from \(10\)--\(12\) components
in one-slot regimes to \(36\)--\(37\) components in full four-slot regimes. This
increase reflects the larger number of jointly varying lexical slots and supports
the interpretation that the reduced space captures the effective local degrees of
semantic variation rather than an arbitrary projection dimension.

The present experiments use the \(90\%\) explained-variance threshold for
selecting the reduced PCA dimension. As a diagnostic check, the same recorded
variance profiles also show that lowering the threshold to \(85\%\) would reduce
the selected dimension by a moderate amount, while preserving the same qualitative
ordering of regimes by complexity. A full sensitivity analysis for higher
thresholds, such as \(95\%\), would require retaining additional principal
components beyond those stored in the present experimental archive and is left
for future work.

The results do not support a rigid hypothesis about the existence of a single universal surface type for all local classes. In many cases, the quadratic model provides a sufficiently strong baseline nonlinear approximation, balancing expressiveness, interpretability, and stability. The cubic model can provide a better fit in more complex cases, but it has a larger number of parameters and is therefore better interpreted as an optional local extension for classes with more complex geometry. The general interpretation is that the class of the fitted model should be selected adaptively according to the local structure of a particular embedding cloud. In the experiments reported here, the fitted quadratic carriers were in almost all cases classified as one-sheeted hyperboloid-type quadrics. This observation suggests that local semantic variation is better viewed as an open curved patch in the reduced representation space than as a closed ellipsoidal cluster. We treat this classification as an additional geometric diagnostic; a detailed linguistic interpretation of quadric types remains a direction for future work.

The comparison with baseline generation methods shows that the surface-based method has an advantage in terms of fitted-geometry fidelity. It provides better agreement with the fitted surface, smaller changes in local second-order shape, and greater stability of the fitted-model coefficients. At the same time, this method does not dominate across all metrics. Neighborhood- and distribution-based criteria measure closeness to the empirical density of the finite cloud. Under such criteria, local perturbation may be preferable, since it generates points directly in small neighborhoods of existing observations. Thus, the second experiment reveals a trade-off between empirical-density fidelity and fitted-geometry fidelity.

Downstream classification clarifies the limits of directly using surface-based generation in classification tasks. In low-data and context-held-out slot classification, the baseline with no synthetic points often remained the best or close to the best. This result is consistent with the geometric experiments, since geometric validity and discriminative utility characterize different properties. A geometrically valid point may agree well with a local manifold-like patch, but provide no benefit to a particular classifier if it does not change or refine the decision boundary.

This result is methodologically consistent with broader observations in the literature on NLP data augmentation. If a model or representation space already has invariance to certain lexical or paraphrase-like changes, additional examples close to existing ones in latent space may be redundant. In our case, we observe a geometric analogue of this effect: surface-based points are correct with respect to the fitted local geometry, but remain in the same local region of representation space and therefore do not necessarily add new discriminative patterns.

The main methodological conclusion is that the proposed approach should primarily be evaluated by whether it provides an explicit local model of semantic variation in embedding space. Such a model can be useful for constructing synthetic latent points, analyzing the structure of representation space, assessing the stability of local semantic neighborhoods, and defining diagnostic criteria for embedding models.

The potential applications of this approach go beyond direct classification. Fitted local surfaces and the corresponding residual-based characteristics can be used for semantic neighborhood expansion, retrieval and reranking, OOD detection, adversarial filtering, uncertainty estimation, hallucination analysis, and the construction of hard positive examples in contrastive learning. In these tasks, it is important to measure, preserve, or control the local geometry of embedding space. For this reason, they are more naturally aligned with the proposed method than the simple task of improving a linear classifier.

The work has several limitations. First, the proposed approach is local in nature. It operates with local controlled semantic patches and does not assume the existence of a single surface for the entire space of texts. Second, the main experiments are conducted on a synthetically constructed template-based dataset. This is a controlled benchmark designed to isolate sources of semantic variation and study the local geometry of embedding clouds. Such control is a methodological advantage for testing the geometric hypothesis, but it also limits the direct transferability of the conclusions to less structured natural corpora.

Third, the results depend on the chosen sentence embedding model. The main
experiments were conducted with
\texttt{sentence-transformers/all-mpnet-base-v2}, so the fitted carriers and
shape-based characteristics should primarily be interpreted as properties of this
representation space. To partially address this limitation, we added an
auxiliary check with \texttt{sentence-transformers/all-MiniLM-L6-v2}. It showed
the same qualitative affine-versus-quadric pattern for the full \(C5\) regimes,
but it does not replace a full cross-model validation. Broader testing across
other embedding families remains future work.

Fourth, downstream evaluation was performed only for one baseline classifier, namely logistic regression. This makes it possible to interpret the results primarily as an effect of the generation method, while reducing the influence of a more complex classification architecture. At the same time, this limits the scope of practical conclusions. Whether the discriminative utility of generated points changes for SVMs, MLPs, or fine-tuned transformer-based models requires separate investigation.

Fifth, the method in its current form is oriented toward offline scenarios. Its computational logic involves forming a local embedding cloud, performing reduced-space fitting, applying surface-based generation, and conducting geometric validity analysis. This format is natural for research, diagnostic, and benchmark scenarios, although it should not be interpreted as a cheap substitute for simple engineering heuristics.

A separate outcome of the work is the \textbf{CoPaGE-300K} dataset (\emph{Controlled Paraphrase Geometry Embeddings}). Its value lies in its controlled template-based structure, slot-level annotation, and precomputed embeddings for reproducible analysis of the local geometry of semantic variation. In this sense, CoPaGE-300K can serve as a benchmark resource for further studies of representation-space geometry, controlled paraphrase variation, and geometry-aware latent generation.

Several directions for future work are natural. First, it is useful to move from a fixed choice of a quadratic or cubic model to an adaptive selection mechanism for the local fitted carrier. Second, the stability of the main geometric patterns should be tested for other embedding models and less controlled paraphrase corpora. Third, fitted-geometry fidelity should be combined with discriminative filtering, so that the selection of synthetic points takes into account not only membership in the fitted surface, but also the potential contribution of the point to the decision boundary. Fourth, the theoretical relation between the local geometry of embedding clouds and invariants of meaning-preserving transformations requires further development.

Thus, the proposed approach is a method for local geometric modeling of semantic variation in sentence embedding space. Its main contribution is to show that controlled semantic variation induces measurable local geometry that can be explicitly approximated, analyzed, and used for geometry-aware latent generation.

%%%%%%%%%%%%%%%%%%%%%%%%%%%%%%%%%%%%%%%%%%%%

\section{Conclusions}

%%%%%%%%%%%%%%%%%%%%%%%%%%%%%%%%%%%%%%%%%%%%

This paper studied the local geometry of embedding clouds arising from controlled
template-based classes of semantically close sentences. The proposed approach is
aimed at explicit modeling of the local structure of semantic variation in
sentence embedding space and at constructing synthetic latent points based on a
fitted local carrier. In this sense, the method should be viewed as an offline
procedure for local geometric analysis, latent probing, and geometry-aware latent
generation.

The experimental results confirm that controlled local classes of semantically
close sentences have measurable nonlinear structure in embedding space. Affine
models often do not adequately describe such embedding clouds, whereas quadratic
models provide a strong baseline nonlinear approximation of the local patch.
Cubic models improve the fit in some cases, but due to their larger number of
parameters they are better interpreted as an optional local extension for more
complex classes. Thus, the natural interpretation is adaptive local geometric
modeling, where the type of fitted carrier is selected according to the structure
of a particular embedding cloud.

A separate result is that local geometry can be used constructively: a fitted
surface can serve as a local carrier for constructing synthetic latent points.
All such computations are performed in a reduced \(r\)-dimensional PCA space,
where the dominant variability of the local cloud is preserved and polynomial
fitting and surface-based projection become computationally more stable.

In terms of fitted-geometry fidelity, the proposed surface-based method has an
advantage over simpler baseline approaches. It better preserves agreement with the fitted surface, Hessian-based
second-order shape descriptors, and the stability of the fitted-model
coefficients. At the same time, neighborhood- and distribution-based
criteria measure a different property: closeness to the empirical density of the
finite cloud. Under such criteria, local perturbation may be preferable, since it
generates points directly in neighborhoods of existing observations. The work
therefore reveals an important trade-off between empirical-density fidelity and
fitted-geometry fidelity.

The downstream experiments showed that geometric validity of synthetic latent
points does not guarantee an automatic improvement in classification
performance. In low-data and context-held-out slot classification, the baseline
with no synthetic points often remained the best or close to the best. This
supports one of the main conclusions of the work: \emph{geometric validity} and
\emph{discriminative utility} are distinct properties. A point may be well
aligned with the fitted local geometry while adding no new information for
constructing the decision boundary. Surface-based generation should therefore be
interpreted primarily as a tool for local geometric modeling and latent probing,
rather than as a universal classification booster.

The scope of the results is shaped by the controlled nature of the data and by
the choice of the embedding model. The main experiments were conducted on
template-based classes, which makes it possible to isolate sources of semantic
variation and study the local geometry of embedding clouds under reproducible
conditions. This control is a methodological advantage for testing the geometric
hypothesis, but it also calls for further evaluation on less structured datasets
that are closer to natural language. In addition, fitted surfaces and
shape-based characteristics should be understood as properties of the data in
a particular representation space; broader cross-model stability of these results remains an open direction for
future research.

A separate outcome of this work is the construction of the \textbf{CoPaGE-300K}
dataset, a controlled
template-based dataset of semantically close sentences with slot-level annotation
and precomputed embeddings. It should not be treated as a natural paraphrase
corpus; its value lies in its controlled structure, full reproducibility, and
suitability for studying the local geometry of semantic variation in sentence
embedding space.

Thus, the main contribution of the work is to show that controlled semantic
variation induces measurable local geometry in embedding space. This geometry can
be explicitly approximated, analyzed, and used for geometry-aware latent
generation. The practical value of the approach lies in controlled modeling of
local semantic neighborhoods and diagnostics of representation space. Promising
directions for future work include adaptive selection of the fitted model,
cross-model stability analysis, application to less controlled corpora,
investigating ways to combine fitted-geometry fidelity with discriminative
filtering, and using local fitted surfaces for OOD detection, adversarial
filtering, uncertainty estimation, hallucination analysis, and the construction
of hard positive examples for contrastive learning.

\appendix
\section{Lexical variants used to construct controlled classes}
\label{app:lexical_variants}

This appendix provides the complete lists of lexical variants used to construct
controlled sets of semantically close sentences. Including these lists is
essential for the reproducibility of the experimental design, since they define
the structure of the local embedding clouds analyzed in the main text.

\subsection{Methodological status of the three template families}

The three template families used in this work have different linguistic status.
Template family A is closest to the classical synonym substitution setting: the
first slot varies reporting verbs, while the remaining slots vary nouns that
define semantically close alternatives within a single political-administrative
template. Template families B and C use a broader class of
\emph{context-compatible semantically close lexical items}. This design makes it
possible to test whether the main geometric regularities are preserved not only
for near-synonym variation, but also for a broader class of controlled lexical
perturbations.

Thus, the results for template family A can be interpreted as the closest to the
synonym substitution regime, whereas template families B and C serve as a
\emph{stress test} for a broader class of semantically close lexical variants.
For this reason, the claims in the main text are formulated for
\emph{controlled template-based local classes of semantically close sentences},
rather than for all synonymic paraphrases in general.

\subsection{Anchor values and subclass construction regimes}

For each template family, the appendix also specifies \emph{anchor values} for
non-varying slots. Anchor lexical items were chosen as neutral, frequent, and
contextually natural representatives of the corresponding lists. Their role is
technical: they fix inactive slots in regimes \(C1\)--\(C4\), so that each
subclass is uniquely defined and reproducible. This work does not include a
separate ablation study of the choice of anchor values; the effect of anchors on
the geometry of the cloud should therefore be treated as a separate question for
future research.

\subsection{Computational scale of the full space}

For each template family, the full four-slot regime \(C5\) contains
$$
18^4 = 104{,}976
$$
unique sentences. Therefore, across the three template families, the full
combined space contains
$$
3\cdot 18^4 = 314{,}928
$$
sentences. This means that encoding all \(C5\) sets in sentence embedding space
is computationally nontrivial, although technically feasible with batch
processing. For this reason, the experimental design must clearly distinguish
between the full combinatorial space and the specific regimes used in individual
experiments.

\subsection{Template family A: institutional communication and policy}
\label{app:template_A}

Base template:
$$
\texttt{The minister \{\} that the \{\} would bring \{\} for \{\}.}
$$

Anchor values for non-varying slots:
$$
(s_1,s_2,s_3,s_4)=(\texttt{said},\texttt{policy},\texttt{benefits},\texttt{citizens}).
$$

\paragraph{Slot \(s_1\): reporting verbs.}
\begin{quote}
said, stated, announced, declared, reported, mentioned, noted, remarked, observed, explained, confirmed, indicated, emphasized, stressed, asserted, claimed, communicated, revealed
\end{quote}

\paragraph{Slot \(s_2\): policy nouns.}
\begin{quote}
reform, policy, measure, initiative, program, plan, proposal, strategy, scheme, package, project, framework, regulation, decision, action, change, legislation, agreement
\end{quote}

\paragraph{Slot \(s_3\): benefit nouns.}
\begin{quote}
benefits, advantages, gains, improvements, support, relief, opportunities, resources, protections, assistance, enhancements, incentives, savings, efficiency gains, better outcomes, new opportunities, additional support, long-term benefits
\end{quote}

\paragraph{Slot \(s_4\): beneficiary nouns.}
\begin{quote}
citizens, residents, families, households, communities, students, workers, teachers, patients, parents, children, consumers, commuters, farmers, taxpayers, businesses, young people, older adults
\end{quote}

\subsection{Template family B: education}

Base template:
$$
\texttt{The teacher designed a \{\} lesson with \{\} exercises for \{\} students in a \{\} course.}
$$

Template family B does not implement a strict synonym substitution regime.
Instead, it uses a broader class of context-compatible semantically close
attributive modifiers. The lists below were deliberately constructed to avoid
literal overlap between slots \(s_1\) and \(s_2\), in order to avoid sentences
such as \texttt{a guided lesson with guided exercises}.

Anchor values for non-varying slots:
$$
(s_1,s_2,s_3,s_4)=(\texttt{practical},\texttt{guided},\texttt{beginner},\texttt{introductory}).
$$

\paragraph{Slot \(s_1\): lesson descriptors.}
\begin{quote}
practical, engaging, interactive, applied, focused, dynamic, coherent, organized, stimulating, accessible, well-paced, informative, balanced, classroom-based, skill-oriented, thoughtful, structured, motivating
\end{quote}

\paragraph{Slot \(s_2\): exercise descriptors.}
\begin{quote}
guided, scaffolded, targeted, incremental, hands-on, practice-based, reinforcing, diagnostic, collaborative, independent, reflective, contextual, problem-solving, stepwise, graded, varied, follow-up, manageable
\end{quote}

\paragraph{Slot \(s_3\): student descriptors.}
\begin{quote}
beginner, novice, entry-level, less-experienced, newly enrolled, first-year, junior, early-stage, foundation-level, developing, emerging, inexperienced, starting, lower-level, pre-intermediate, first-term, introductory-level, initial-stage
\end{quote}

\paragraph{Slot \(s_4\): course descriptors.}
\begin{quote}
introductory, foundational, elementary, basic, preparatory, survey, core, entry-course, lower-division, initial, starting-level, general, primer, baseline, bridge, orientation, gateway, first-cycle
\end{quote}

\subsection{Template family C: medicine}
\label{app:template_C}

Base template:

\texttt{The doctor recommended a \{\} treatment with \{\} monitoring for \{\} patients} \\
\texttt{during \{\} recovery.}

Template family C, like family B, goes beyond strict synonym substitution and
uses a broader class of context-compatible semantically close lexical items. In
constructing the lists, the literal repetition of \texttt{follow-up} in two
different slots was deliberately removed, in order to avoid tautological
combinations such as \texttt{follow-up treatment with follow-up monitoring}.

Anchor values for non-varying slots:
$$
(s_1,s_2,s_3,s_4)=(\texttt{conservative},\texttt{regular},\texttt{adult},\texttt{early}).
$$

\paragraph{Slot \(s_1\): treatment descriptors.}
\begin{quote}
conservative, targeted, supportive, individualized, stepwise, noninvasive, standard, evidence-based, supervised, outpatient, low-intensity, short-term, structured, moderate, symptom-focused, regimen-based, gradual, protocol-driven
\end{quote}

\paragraph{Slot \(s_2\): monitoring descriptors.}
\begin{quote}
regular, continuous, close, periodic, ongoing, systematic, daily, weekly, scheduled, attentive, remote, clinical, proactive, routine, longitudinal, bedside, sensor-based, post-discharge
\end{quote}

\paragraph{Slot \(s_3\): patient descriptors.}
\begin{quote}
adult, elderly, vulnerable, high-risk, postoperative, chronic, ambulatory, frail, stable, symptomatic, referred, monitored, at-risk, immunocompromised, long-term-care, recovering, geriatric, working-age
\end{quote}

\paragraph{Slot \(s_4\): recovery descriptors.}
\begin{quote}
early, initial, gradual, assisted, home-based, extended, planned, safe, partial, staged, steady, prolonged, post-acute, supported, late-phase, convalescent, rehabilitative, stabilizing
\end{quote}

\subsection{Correspondence between template families and experimental regimes}
\label{app:class_mapping}

For each of the three template families, the same sequence of regimes
\(C1\)--\(C5\) is constructed:
\begin{gather*}
\begin{aligned}
C1 &: \text{only slot } s_1 \text{ varies}, \quad |C1| = 18,\\
C2 &: \text{only slot } s_2 \text{ varies}, \quad |C2| = 18,\\
C3 &: \text{slots } s_1 \text{ and } s_2 \text{ vary jointly}, \quad |C3| = 18^2 = 324,\\
C4 &: \text{slots } s_1,s_2,s_3 \text{ vary jointly}, \quad |C4| = 18^3 = 5{,}832,\\
C5 &: \text{slots } s_1,s_2,s_3,s_4 \text{ vary jointly}, \quad |C5| = 18^4 = 104{,}976.
\end{aligned}
\end{gather*}

To avoid ambiguity, the notation \(A\!\!-\!C1,\dots,A\!\!-\!C5\),
\(B\!\!-\!C1,\dots,B\!\!-\!C5\), and \(C\!\!-\!C1,\dots,C\!\!-\!C5\) is used
when necessary. Thus, \(C1\) and \(C2\) are not independent ``small classes'',
but two single-slot control regimes within each template family.

\subsection{Reproducibility note}

All lists of lexical variants given in this appendix must match the lists
actually used in the code for generating controlled sentences. This is important
because even a small change in the lexical composition may affect the geometry of
the corresponding embedding cloud, fitted local models, shape-based
characteristics, coefficient consistency, and downstream results.

\section{Auxiliary cross-model robustness check with \texttt{all-MiniLM-L6-v2}}
\label{app:minilm_robustness}

This appendix reports an auxiliary cross-model robustness check performed with
\texttt{sentence-transformers/all-MiniLM-L6-v2}. The main experiments in the
paper use \texttt{sentence-transformers/all-mpnet-base-v2}, which produces
\(768\)-dimensional sentence embeddings. The purpose of the additional check is
to test whether the main qualitative geometric observation persists when a
different sentence-transformer model with a smaller embedding dimension is used.

The \texttt{all-MiniLM-L6-v2} model produces \(384\)-dimensional sentence
embeddings. We do not treat this auxiliary experiment as a full repetition of all
main experiments. Instead, it is used as a compact cross-model check of the
central geometric claim: local embedding clouds induced by controlled semantic
variation are better described by nonlinear fitted carriers than by affine
models.

\subsection{Adaptive PCA dimensionality for \texttt{all-MiniLM-L6-v2}}

As in the main experiments, PCA was applied separately to each controlled local
class. The reduced dimension \(r\) was selected by the \(90\%\)
explained-variance criterion. Table~\ref{tab:minilm_pca_dimensions} reports the
selected dimensions for all template families and regimes.

\begin{table}[ht]
\centering
\caption{Adaptive PCA dimensionality for \texttt{all-MiniLM-L6-v2} selected by the \(90\%\) explained-variance criterion.}
\label{tab:minilm_pca_dimensions}
\begin{tabular}{lrrrr}
\hline
\textbf{Regime} & \textbf{\(N\)} & \textbf{Original dim.} & \textbf{\(r\)} & \textbf{Explained variance} \\
\hline
A-C1 & 18      & 384 & 10 & 0.907445 \\
A-C2 & 18      & 384 & 13 & 0.914422 \\
A-C3 & 324     & 384 & 19 & 0.912216 \\
A-C4 & 5,832   & 384 & 28 & 0.906483 \\
A-C5 & 104,976 & 384 & 38 & 0.901610 \\
\hline
B-C1 & 18      & 384 & 14 & 0.928656 \\
B-C2 & 18      & 384 & 13 & 0.921866 \\
B-C3 & 324     & 384 & 24 & 0.902640 \\
B-C4 & 5,832   & 384 & 32 & 0.907814 \\
B-C5 & 104,976 & 384 & 39 & 0.903656 \\
\hline
C-C1 & 18      & 384 & 13 & 0.907216 \\
C-C2 & 18      & 384 & 10 & 0.908872 \\
C-C3 & 324     & 384 & 22 & 0.904844 \\
C-C4 & 5,832   & 384 & 30 & 0.901122 \\
C-C5 & 104,976 & 384 & 38 & 0.902562 \\
\hline
\end{tabular}
\end{table}

The selected dimensions range from \(r=10\) to \(r=39\). Although the original
embedding dimension of \texttt{all-MiniLM-L6-v2} is only half that of
\texttt{all-mpnet-base-v2}, the effective local dimensions of the full
four-slot regimes remain comparable. In particular, the \(C5\) regimes require
\(38\), \(39\), and \(38\) principal components for template families \(A\),
\(B\), and \(C\), respectively. This indicates that the controlled semantic
variation is still distributed across several dozen effective local directions
in the MiniLM embedding space.

\subsection{Affine and quadratic fitting in a fixed reduced space}

For computational reasons, the auxiliary MiniLM fitting check was performed in a
fixed reduced space of dimension \(r=20\). This setting is intentionally lighter
than the full \(90\%\)-variance reduced spaces reported in
Table~\ref{tab:minilm_pca_dimensions}. The \(r=20\) truncation preserves
\(0.6978\), \(0.6791\), and \(0.7005\) of the total PCA variance for
\(A\text{-}C5\), \(B\text{-}C5\), and \(C\text{-}C5\), respectively, whereas the
\(90\%\)-variance criterion selects \(r=38\), \(r=39\), and \(r=38\). 

Therefore, this check is deliberately more compressed than the main
\(90\%\)-variance setting, and the results below should not be interpreted as a
full repetition of the main fitting experiment. Their role is to test whether
the affine-versus-quadratic contrast persists under a different embedding model
and a smaller reduced coordinate space.

For each of the three full four-slot regimes \(A\text{-}C5\), \(B\text{-}C5\),
and \(C\text{-}C5\), we fitted an affine model and a quadric implicit model in
the \(r=20\) reduced space. The fitted residuals were evaluated using the
normalized implicit surface residual
$$
\rho(z)
=
\frac{|f(z)|}{\|\nabla f(z)\|+\varepsilon}.
$$
This quantity is a gradient-normalized algebraic residual and should not be
interpreted as an exact orthogonal distance to the fitted surface.

\begin{table}[ht]
\centering
\caption{Auxiliary affine-versus-quadric fitting check for \texttt{all-MiniLM-L6-v2} in a fixed reduced space \(r=20\). Surface RMSE and MAE are computed from the normalized implicit surface residual.}
\label{tab:minilm_affine_quadric_r20}
\begin{tabular}{llrrrrr}
\hline
\textbf{Regime} & \textbf{Model} & \textbf{Degree} & \textbf{Parameters} &
\textbf{Alg. RMSE} & \textbf{Surface RMSE} & \textbf{Surface MAE} \\
\hline
A-C5 & Affine  & 1 & 21  & 0.650885 & 0.732287 & 0.603730 \\
A-C5 & Quadric & 2 & 231 & 0.090136 & 0.114082 & 0.086731 \\
\hline
B-C5 & Affine  & 1 & 21  & 0.456916 & 0.482266 & 0.402423 \\
B-C5 & Quadric & 2 & 231 & 0.066729 & 0.082574 & 0.062777 \\
\hline
C-C5 & Affine  & 1 & 21  & 0.512891 & 0.550278 & 0.449854 \\
C-C5 & Quadric & 2 & 231 & 0.063231 & 0.074483 & 0.057713 \\
\hline
\end{tabular}
\end{table}

The results in Table~\ref{tab:minilm_affine_quadric_r20} show the same
qualitative pattern as the main experiments. In all three \(C5\) regimes, the
quadric model substantially reduces both algebraic RMSE and normalized implicit
surface RMSE relative to the affine model. For example, the surface RMSE
decreases from \(0.732287\) to \(0.114082\) for \(A\text{-}C5\), from
\(0.482266\) to \(0.082574\) for \(B\text{-}C5\), and from \(0.550278\) to
\(0.074483\) for \(C\text{-}C5\).

Thus, even in a deliberately reduced \(r=20\) coordinate space and under a
different sentence-transformer embedding model, the affine approximation remains
insufficient for the full controlled semantic clouds, whereas the quadric fitted
carrier captures the local structure much more accurately. This auxiliary check
supports the qualitative robustness of the main geometric conclusion: the
nonlinear local structure observed in the primary experiments is not merely an
artifact of the original \(768\)-dimensional MPNet-based embedding space.

At the same time, this check has a limited scope. It does not replace a full
cross-model validation over many embedding architectures, and it does not repeat
all geometric-validity and downstream experiments. Broader cross-model analysis,
including other sentence embedding families and natural paraphrase corpora,
remains a direction for future work.

\section{Proof of the barycentric angular stability theorem}
\label{app:barycentric_stability}

This appendix proves the angular stability statement used in
Section~\ref{sec:methodology}. Let
$$
Z=\{z_1,\dots,z_N\}\subset\mathbb{R}^r
$$
be a finite cloud of points, and let \(a\in\mathbb{R}^r\) be a reference point.
We define
\begin{gather*}
D_Z=\operatorname{diam}(Z)
=
\max_{i,j}\|z_i-z_j\|,
\\
d_Z(a)=\operatorname{dist}(a,\operatorname{conv}(Z)).
\end{gather*}
Here \(D_Z\) is the Euclidean diameter of the cloud, and \(d_Z(a)\) is the
distance from the reference point \(a\) to the convex hull of the cloud.

\begin{theorem}
Let \(Z=\{z_1,\dots,z_N\}\subset\mathbb{R}^r\) be a finite cloud and let
\(a\in\mathbb{R}^r\) be a reference point such that
$$
d_Z(a)=\operatorname{dist}(a,\operatorname{conv}(Z))>0.
$$
Let \(v\in\operatorname{conv}(Z)\) be a barycentric point such that \(v\ne a\).
If
$$
D_Z\leq 2d_Z(a),
$$
then, for every \(j=1,\dots,N\),
$$
\angle(v-a,z_j-a)
\leq
2\arcsin\frac{D_Z}{2d_Z(a)}.
$$
\end{theorem}

\begin{proof}
Since \(v\in\operatorname{conv}(Z)\), the point \(v\) lies in the convex hull of
the original cloud. The diameter of the convex hull of a finite set is equal to
the diameter of the set itself. Hence, for every \(j=1,\dots,N\),
$$
\|v-z_j\|\leq D_Z.
$$
Indeed, each point of \(\operatorname{conv}(Z)\) is a convex combination of the
points \(z_1,\dots,z_N\), and the convex hull cannot have a larger Euclidean
diameter than the maximum distance between its generating points.

Next, from the definition of
$$
d_Z(a)=\operatorname{dist}(a,\operatorname{conv}(Z)),
$$
every point of \(\operatorname{conv}(Z)\) is at distance at least \(d_Z(a)\) from
the reference point \(a\). Since
$$
v\in\operatorname{conv}(Z),
\qquad
z_j\in Z\subset\operatorname{conv}(Z),
$$
we have
\begin{gather}
\|v-a\|\geq d_Z(a),
\\
\|z_j-a\|\geq d_Z(a).
\end{gather}

Let
$$
\theta_j=\angle(v-a,z_j-a).
$$
For any two nonzero vectors \(u,w\in\mathbb{R}^r\) with angle
\(\theta=\angle(u,w)\), the following elementary Euclidean estimate holds:
$$
\|u-w\|
\geq
2\min(\|u\|,\|w\|)\sin\frac{\theta}{2}.
$$
Applying this estimate to
$$
u=v-a,
\qquad
w=z_j-a,
$$
and using
$$
\min(\|v-a\|,\|z_j-a\|)\geq d_Z(a),
$$
we obtain
$$
\|(v-a)-(z_j-a)\|
\geq
2d_Z(a)\sin\frac{\theta_j}{2}.
$$
Since
$$
\|(v-a)-(z_j-a)\|=\|v-z_j\|,
$$
and \(\|v-z_j\|\leq D_Z\), it follows that
$$
2d_Z(a)\sin\frac{\theta_j}{2}
\leq
D_Z.
$$
Therefore,
$$
\sin\frac{\theta_j}{2}
\leq
\frac{D_Z}{2d_Z(a)}.
$$
Since \(D_Z\leq 2d_Z(a)\), the right-hand side is at most \(1\), and hence
$$
\frac{\theta_j}{2}
\leq
\arcsin\frac{D_Z}{2d_Z(a)}.
$$
Thus,
$$
\theta_j
\leq
2\arcsin\frac{D_Z}{2d_Z(a)}.
$$
This proves the claim.
\end{proof}

\begin{corollary}
Under the assumptions of the theorem, for every barycentric point
\(v\in\operatorname{conv}(Z)\) such that \(v\ne a\),
$$
\alpha_{Z\cup\{v\}}(a)
\leq
\max\left\{
\alpha_Z(a),\;
2\arcsin\frac{D_Z}{2d_Z(a)}
\right\},
$$
where
$$
\alpha_Z(a)
=
\max_{i,j}\angle(z_i-a,z_j-a)
$$
is the angular diameter of the original cloud with respect to the reference
point \(a\).
\end{corollary}

\begin{proof}
The pairwise angles between the original points of \(Z\), measured with respect
to the reference point \(a\), are bounded by \(\alpha_Z(a)\). By the theorem,
every angle between \(v-a\) and \(z_j-a\) is bounded by
$$
2\arcsin\frac{D_Z}{2d_Z(a)}.
$$
Therefore, the largest angle in the enlarged cloud \(Z\cup\{v\}\), measured with
respect to \(a\), is bounded by the maximum of these two quantities.
\end{proof}

\begin{corollary}
If, in addition,
$$
2\arcsin\frac{D_Z}{2d_Z(a)}
\leq
\alpha_Z(a),
$$
then every barycentric point \(v\in\operatorname{conv}(Z)\), \(v\ne a\), does
not increase the angular diameter of the cloud with respect to \(a\):
$$
\alpha_{Z\cup\{v\}}(a)\leq \alpha_Z(a).
$$
\end{corollary}

\begin{proof}
This follows immediately from the previous corollary.
\end{proof}

The theorem and its corollaries show that convex barycentric initialization is
automatically internal in the affine sense, while angular stability requires an
additional separation condition relative to a chosen reference point. If the
convex hull of the local cloud is well separated from \(a\) relative to its
diameter, then barycentric points remain within a controlled angular region
generated by the original cloud.

In the standard centered PCA coordinates used in the implementation, the origin
coincides with the empirical mean of the reduced cloud. Since this mean is an
equal-weight convex combination of the points, the origin belongs to
\(\operatorname{conv}(Z)\). Therefore, the separation condition generally does
not hold for \(a=0\) in centered PCA coordinates. The theorem should therefore
be interpreted as a sufficient angular-stability statement relative to an
external reference point, not as an automatic guarantee for centered PCA
representations. In the implemented algorithm, barycentric initialization
guarantees convex-hull internality, while angular, neighborhood, and
distributional admissibility are assessed empirically by the geometric validity
criteria.
\end{document}